\documentclass[11pt]{article}

\usepackage[final]{acl}

\usepackage{times}
\usepackage{latexsym}

\usepackage[T1]{fontenc}

\usepackage[utf8]{inputenc}

\usepackage{microtype}

\usepackage{inconsolata}

\usepackage{graphicx}
\usepackage{amsmath}
\usepackage{amssymb}
\usepackage{bbold}
\usepackage[table]{xcolor}
\usepackage{array}
\usepackage{arydshln}
\usepackage{caption}
\usepackage{subcaption}
\usepackage{booktabs}
\usepackage{tabularx}
\usepackage{algorithm}
\usepackage{algpseudocode}
\usepackage{amsthm}
\usepackage{marvosym}
\newtheorem{proposition}{Proposition}

\definecolor{LightPurple}{HTML}{D6D6FF}

\newcommand{\hhh}{C\textsc{ram}}

\title{\hhh: Centroid-Routing and Adaptive MoE for\\ Multimodal Continual Instruction Tuning}
\author{Jun-Tao Tang\textsuperscript{* 2},
        Zhen-Hao Xie\textsuperscript{* 1,2},
        Yu-Cheng Shi\textsuperscript{2},
        Da-Wei Zhou\textsuperscript{\Letter 1,2} \\
         School of Artificial Intelligence, Nanjing University \\
        State Key Laboratory of Novel Software Technology, Nanjing University \\
        \textsuperscript{*} Equal contribution   \quad \textsuperscript{\Letter} Corresponding author\\
         {\tt \small
        \{juntao.tang, 231250034\}@smail.nju.edu.cn, 
        \{wenzh,zhoudw\}@lamda.nju.edu.cn
        }
}

\begin{document}
\maketitle
\begin{abstract}
Multimodal Large Language Models (MLLMs) unify heterogeneous vision-language tasks under a shared generative framework via instruction tuning, yet real-world deployment demands continuous capability expansion, making Multimodal Continual Instruction Tuning (MCIT) essential. Existing methods either update all tasks with a shared parameter set or allocate dedicated modules for each new task. Shared updates force heterogeneous tasks to compete, causing forgetting of learned capabilities. Conversely, isolated expansion prevents interference but severely limits parameter efficiency over long task streams. To address this dilemma, we propose \hhh{} (\underline{C}entroid-\underline{R}outing and \underline{A}daptive \underline{M}oE). Specifically, by isolating task-specific patterns into independent modules, \hhh{} mitigates catastrophic forgetting across tasks. To further boost parameter efficiency, we utilize adaptive-rank instantiation to identify the capability gap between existing expert capability and new task demands, and dynamically allocate only the necessary parameters. To ensure stable reuse among tasks, centroid-guided routing recognizes and activates existing experts' capabilities, while an orthogonality penalty confines new updates to task-specific directions, preventing re-learning general capability. Extensive experiments across diverse benchmarks demonstrate its superiority over existing methods. Code is available at \url{https://github.com/LAMDA-CL/EMNLP2026-CRAM}.

\end{abstract}
\begin{figure}[t]
  \centering
  \footnotesize
  \captionsetup[sub]{font=small}  
  \begin{subfigure}[b]{0.495\columnwidth}
    \centering
    \includegraphics[width=\linewidth]{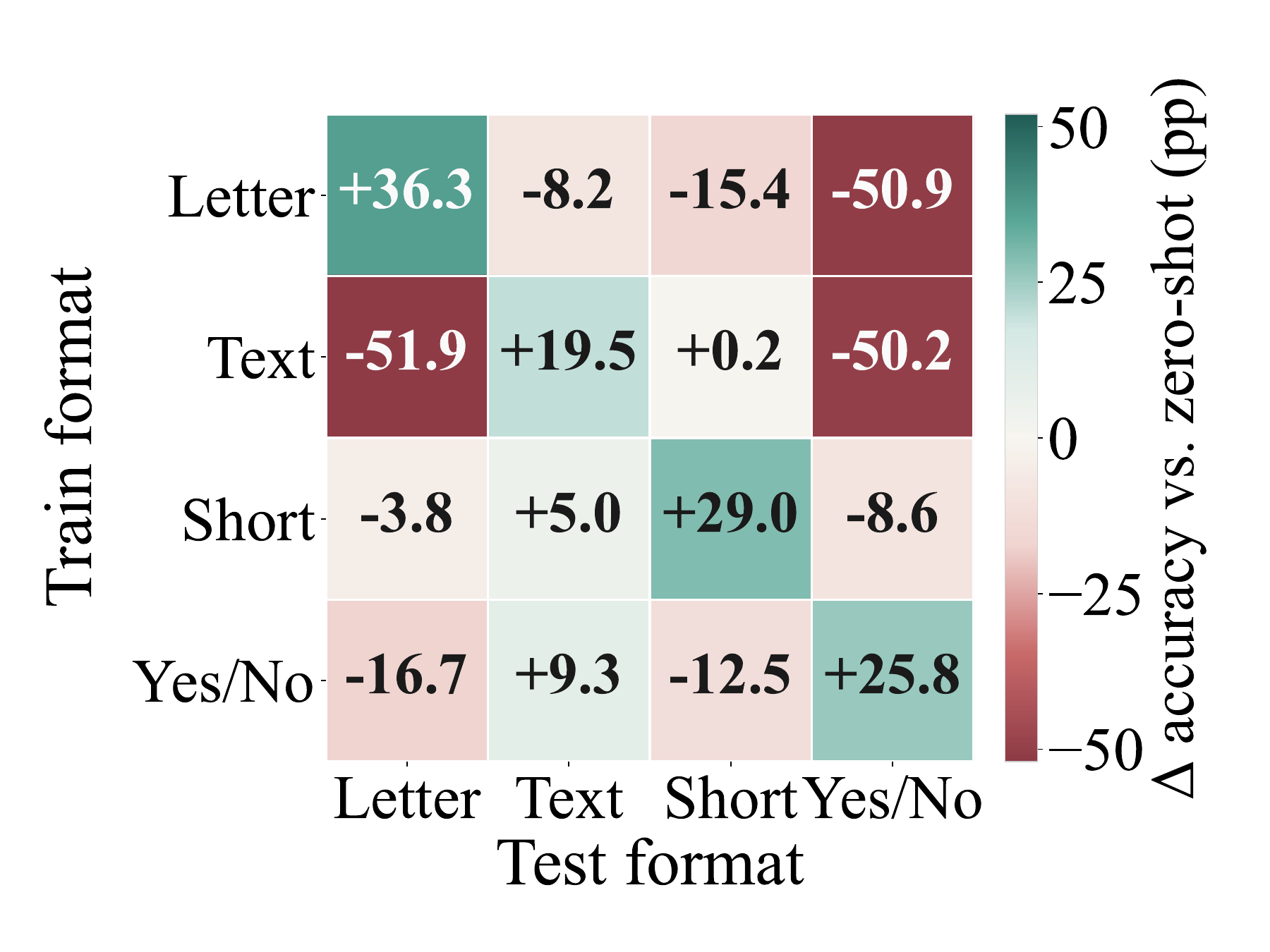}
    \caption{Change vs. Zero-shot}
    \label{fig:pre1_1}
  \end{subfigure}
  \hfill
  \begin{subfigure}[b]{0.495\columnwidth}
    \centering
    \includegraphics[width=\linewidth]{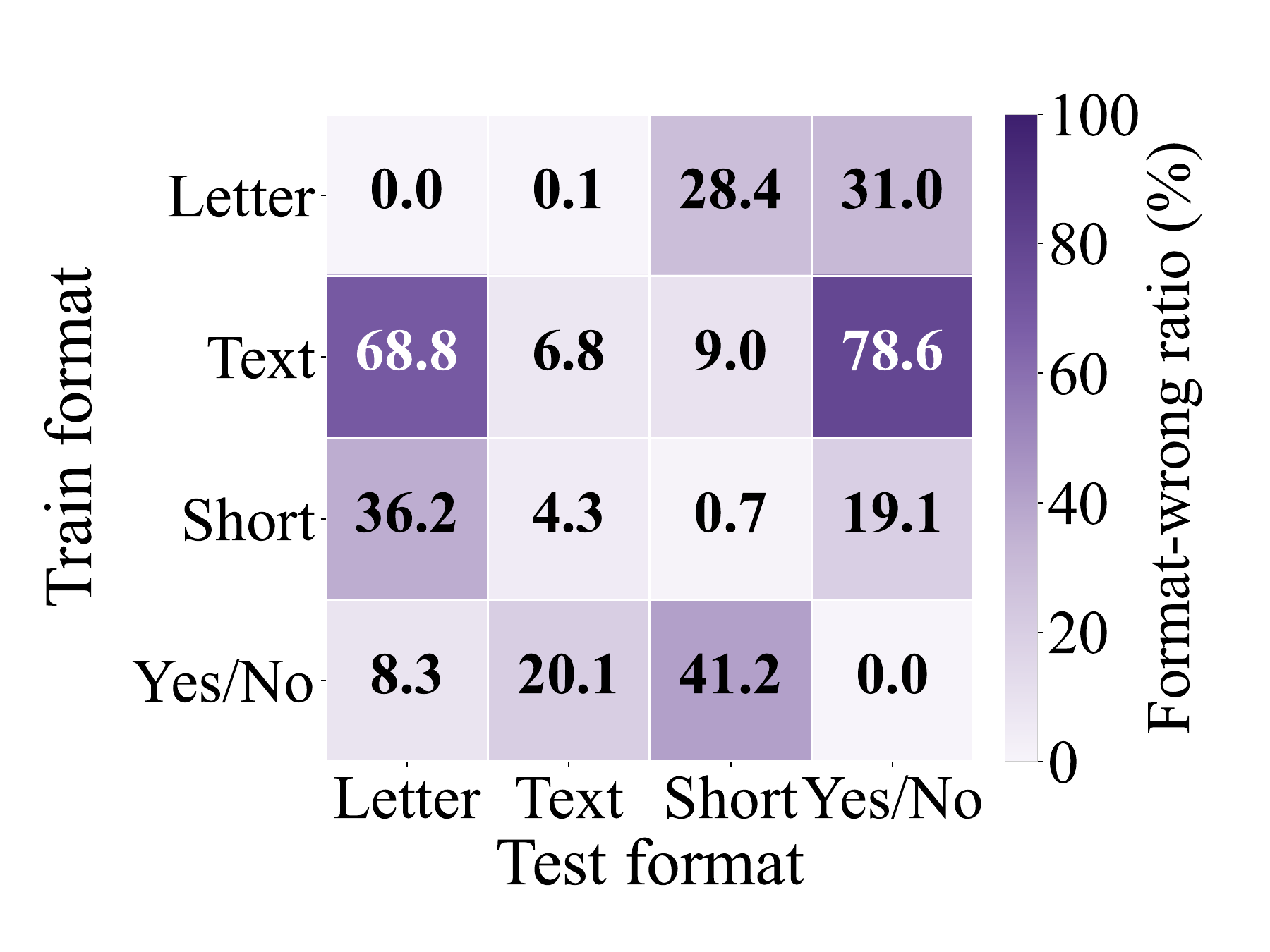}
    \caption{Format-wrong ratio}
    \label{fig:pre1_2}
  \end{subfigure}
    \caption{Format-level interference. (a) Accuracy change relative to zero-shot for each train-test format pair. (b) Ratio of format wrong but semantically correct responses among errors.}
  \label{fig:pre1}
  \vspace{-4mm}
\end{figure}

\begin{figure}[t]
  \centering
  \includegraphics[width=0.85\columnwidth]{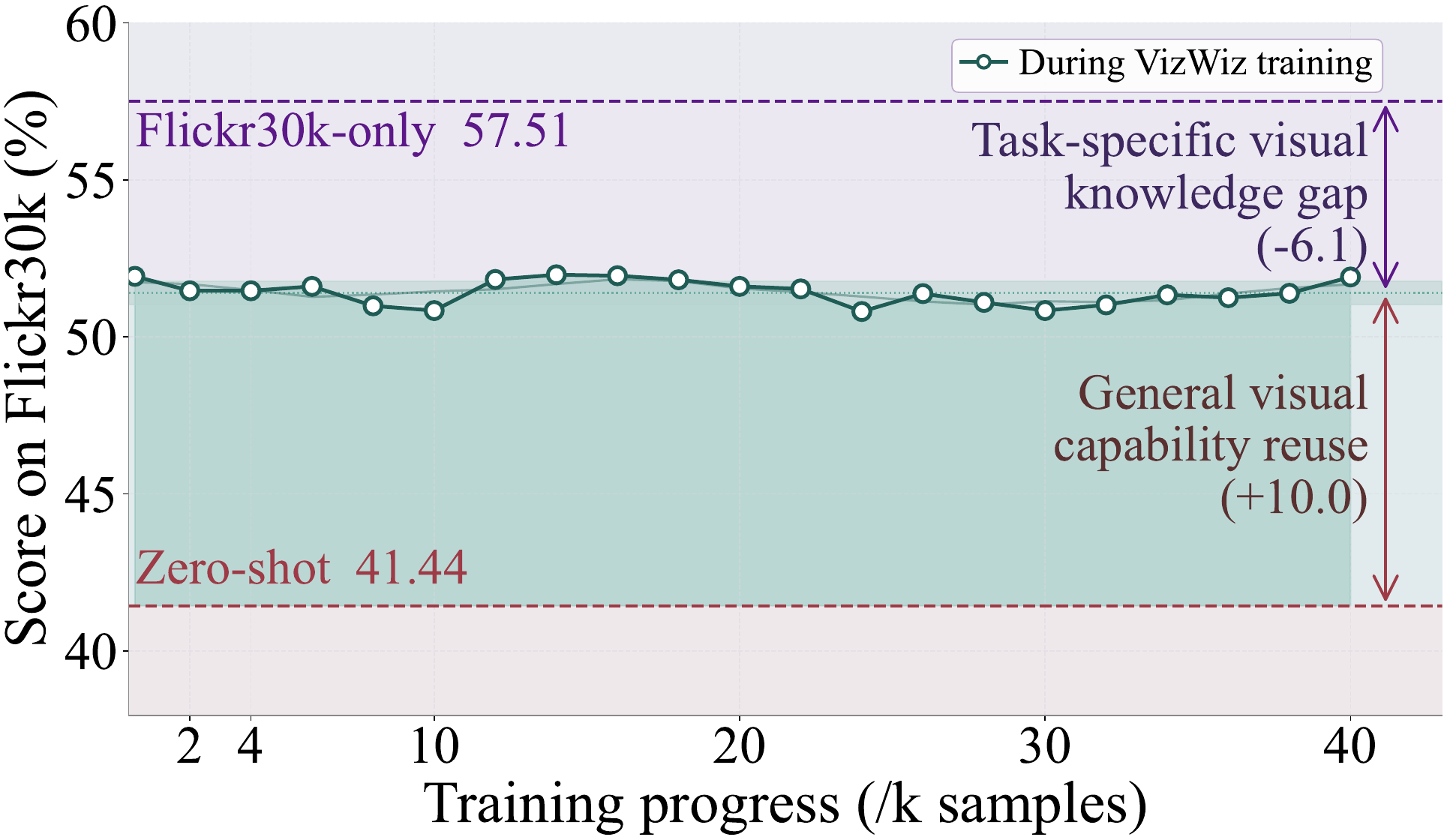}
  \caption{Capability transfer during training.}
  \label{fig:pre2}
  \vspace{-5mm}
\end{figure}

\section{Introduction}
Recently, multimodal Large Language Models (MLLMs)~\citep{liu2023llava,Qwen-VL,zhu2023minigpt} have achieved remarkable capabilities through multimodal instruction tuning~\citep{tong2025metamorph,instrcut2}. By reformulating heterogeneous vision-language problems into a unified generative framework, instruction tuning enables MLLMs to generalize across diverse downstream applications~\citep{dai2023instructblip,instruct1}, from visual reasoning to document understanding. While standard MLLM optimization relies on a static multitask regime that assumes joint data availability~\citep{instruct3,joint_train}, real-world deployment rarely satisfies this premise: tasks arrive sequentially, inherently introducing distribution shifts across both instructional formats and visual domains. As a result, a practically useful MLLM must continually acquire new capabilities, yet sequential adaptation inherently risks catastrophic forgetting~\citep{zhou2024continual,kirkpatrick2017overcoming,qi2025adaptive}, where gradient updates for new tasks overwrite established representations and degrade historical performance~\citep{rebuffi2017icarl}. To enable continuous capability acquisition while mitigating this degradation, multimodal continual instruction tuning (MCIT)~\citep{chen2024coin,huo2025continue,hu2026dynamic} has emerged as a critical paradigm.

Existing MCIT methods mainly consider shared adaptation or parameter isolation. 
Shared approaches~\citep{chen2025sefe,xie2026same,chen2024coin,huai2025cl,zhao2025llava} update all tasks within a single parameter set, causing new updates to overwrite learned patterns. Conversely, isolation methods~\citep{guo2025federated,zeng2025modalprompt,guo2025hide} assign separate modules to each task, leading to unbounded parameter growth and preventing general ability from being shared across tasks. Nevertheless, they overlook a key balance: task-specific patterns must be isolated to avoid interference, while general knowledge should be preserved for reuse. 

To identify what requires isolation and what can be shared, we design controlled experiments that fix one modality while varying the other.
First, to isolate the impact of instruction formats, we use the ArxivQA~\citep{li2024arxiv} dataset to generate four response variants by altering the instruction-response convention. For example, we convert a multiple-choice question into a short-answer format by extracting the correct option's content and adjusting the prompt accordingly (see Appendix~\ref{app:pre1} for details). We then fine-tune the model on one variant and evaluate it across the others. As shown in Fig.~\ref{fig:pre1_1}, training on any single format severely degrades performance on the rest: the model generates semantically accurate answers but fails to adhere to the required output conventions (Fig.~\ref{fig:pre1_2}). This confirms that heterogeneous instruction formats inherently trigger severe interference, acting as a primary factor of catastrophic forgetting in shared-parameter regimes.

We further examine how visual representations behave across distinct domains. Specifically, we select two datasets with the same answer format but non-overlapping visual content: VizWiz-Caption~\citep{gurari2018vizwiz} and Flickr30k~\citep{plummer2015flickr30k}. Throughout the training process on VizWiz-Caption, we monitor the model's performance on Flickr30k. As shown in Fig.~\ref{fig:pre2}, the accuracy stabilizes above the zero-shot baseline but remains below direct fine-tuning on Flickr30k. These results confirm that visual capabilities are partly transferable across tasks, while task-specific ability learning remains necessary. Critically, this reveals a fundamental limitation of isolation methods: by segregating tasks into independent modules, such segregation inherently blocks the reuse of transferable capabilities, forcing redundant learning and limiting parameter efficiency.

Collectively, these findings demonstrate that instruction conventions are highly sensitive patterns requiring isolation, whereas visual capabilities are partially transferable and thus can be reused. 
Guided by this insight, we propose \hhh{} (\underline{C}entroid-\underline{R}outing and \underline{A}daptive \underline{M}oE). 
It groups instructions by similarity into separate clusters, isolating conflicting updates to prevent format interference. 
Within each cluster, adaptive-rank instantiation filters out directions already mastered by historical experts, allocating parameters strictly to unmet capability gaps, naturally bounding model growth. 
Centroid-guided routing further stabilizes expert reuse by directing inputs to experts with aligned capabilities, while an orthogonality penalty confines new updates to missing patterns, preventing them from overwriting historical knowledge.
Together, \hhh{} achieves state-of-the-art performance by balancing format isolation with visual reuse, without compromising parameter efficiency.

\section{Related Work}
\label{sec:relatedwork}
Continual instruction tuning~\citep{chen2024coin} is essential for deploying MLLMs in open-world settings~\citep{liu2025continual}. Existing methods predominantly adopt one of two structural paradigms: \textit{shared adaptation} consolidates trainable parameters into a unified subspace (e.g., LoRA/MoE-LoRA)~\citep{hu2022lora,chen2024coin,zhao2025llava} to maximize plasticity~\citep{plasticity}, often supplemented by replay~\citep{rebuffi2017icarl,li2025multimodal,marouf2025ask} or regularization~\citep{xie2026same,chen2025sefe,huai2025cl}, yet inevitably entangles heterogeneous instruction patterns and triggers gradient interference~\citep{wang2023orthogonal,cao2024continual}; \textit{parameter isolation} mitigates interference by allocating task-dedicated modules~\citep{guo2025hide,zeng2025modalprompt}, preserving stability but fracturing reusable visual competencies and incurring linear parameter growth~\citep{guo2025federated}. Despite their divergent implementations, both paradigms frame sharing and isolation as mutually exclusive. This overlooks the heterogeneous structure of multimodal knowledge~\citep{chen2025sefe}, where some instructional patterns require isolation while transferable visual representations should be systematically reused~\citep{wang2025smolora}.

\begin{figure*}[t]
    \centering
    \includegraphics[width=\linewidth]{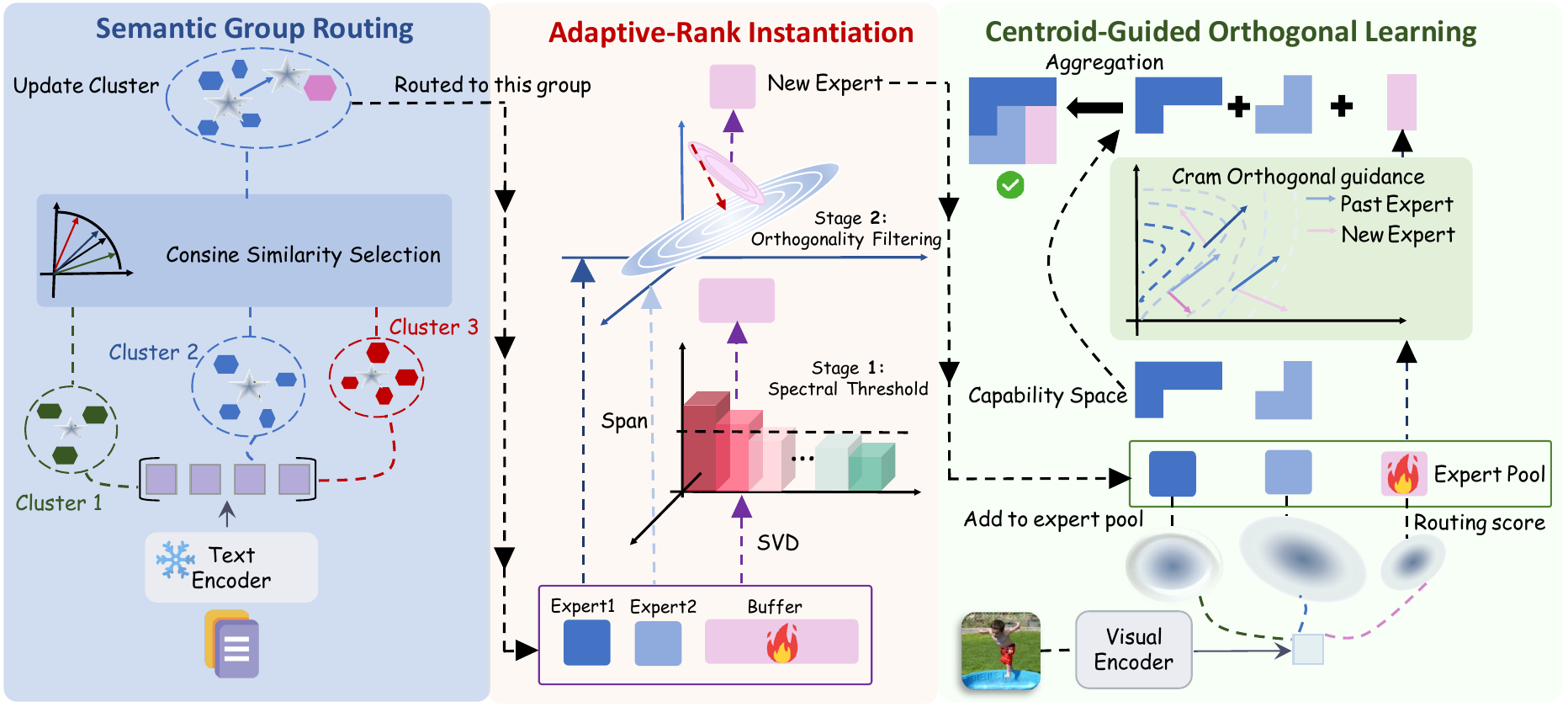}
    \vspace{-8mm}
    \caption{Overview of \hhh. \hhh~decouples via (i) \textit{Semantic Group Routing} that groups instructions by semantic similarity to protect output conventions, (ii) \textit{Adaptive-Rank Instantiation} that allocates parameters exclusively to capability gaps via orthogonality filtering, and (iii) \textit{Centroid-Guided Orthogonal Learning} to stabilize expert activation while confining new updates to orthogonal subspaces.}
    \vspace{-5mm}
\end{figure*}

\section{Preliminaries}
\label{sec:preliminaries}
\noindent\textbf{Multimodal Continual Instruction Tuning (MCIT).} 
We consider an MLLM~\citep{liu2023llava} comprising a vision encoder $\phi(\cdot)$, a multimodal projector $\pi(\cdot)$, and an LLM backbone $f(\cdot)$. To enable continual adaptation~\citep{chen2024coin}, we inject lightweight modules into the LLM, with trainable parameters denoted as $\theta$. Let $\{\mathcal{D}_t\}_{t=1}^T$ denote a sequential task stream, where $\mathcal{D}_t = \{(v_i, q_i, y_i)\}_{i=1}^{n_t}$ contains images $v_i$, instructions $q_i$, and response sequences $y_i = (y_{i,1}, \dots, y_{i,L_i})$. Each input is encoded as $z_i = [\pi(\phi(v_i)); \text{Emb}(q_i)]$, where $\text{Emb}(\cdot)$ is the token embedding layer of $f$. The model generates responses autoregressively, factorizing the conditional probability as $p_{\theta}(y_i \mid z_i) = \prod_{j=1}^{L_i} p_{\theta}(y_{i,j} \mid z_i, y_{i,<j})$.
The negative log-likelihood loss is defined as:
\begin{equation}
\mathcal{L}_{\text{NLL}} = -\sum_{j=1}^{L_i} \log p_{\theta}(y_{i,j} \mid z_i, y_{i,<j}).
\label{eq:nll_loss}
\end{equation}
The learning objective at task $t$ is formulated as:
\begin{equation}
\nonumber
\theta_t^* = \arg\min_{\theta} \mathbb{E}_{(v_i,q_i,y_i)\sim \mathcal{D}_{\le t}} \big[ \mathcal{L}_{\text{NLL}}(\theta; z_i, y_i) \big],
\label{eq:mcit_objective}
\end{equation}
where $\mathcal{D}_{\le t} = \bigcup_{k=1}^{t} \mathcal{D}_k$ denotes all data observed up to task $t$. In rehearsal-free MCIT setting~\citep{zhou2024continual,chen2024coin}, we can only access $\mathcal{D}_t$ when learning task $t$. This mismatch between the global cumulative objective and the locally available data inherently drives catastrophic forgetting~\citep{kirkpatrick2017overcoming}.

\noindent\textbf{MoE with LoRA Experts.}
In conventional MoE-LoRA designs, the backbone remains frozen while a pool of trainable LoRA experts is injected into specific modules~\citep{hu2022lora}. Let $\mathbf{W_0}$ denote the frozen weight and $\Delta \mathbf{W_i} = \mathbf{B_i A_i}$ be the low-rank update for the $i$-th expert. Given an input $h \in \mathbb{R}^d$ at a layer, the layer output is computed as:
\begin{equation}
\nonumber
h' = \mathbf{W_0} h + \sum_{i=1}^{N} \omega_i B_i A_i h,
\label{eq:moe_lora}
\end{equation}
where $\omega_i = \mathrm{Softmax}_\tau(\mathbf{W_G} h)_i$ denotes the routing weight for expert $i$. Here, $\mathbf{W_G} \in \mathbb{R}^{N \times d}$ is the linear gating matrix and $\tau$ is the temperature.

\noindent\textbf{Discussion.}
While effective in static environments, the linear gating degrades under continual instruction tuning with multiple tasks, since a single router is forced to simultaneously satisfy conflicting format isolation and visual reuse objectives.

\section{Method}
\label{sec:method}
To address the isolation-sharing dilemma, we follow a straightforward intuition: by isolating conflicting instruction formats first, we can safely reuse general ability and focus parameter allocation solely on the specific gaps that truly need them.
\subsection{Group-Isolated Semantic Routing}
\label{subsec:group_routing}
In a continual instruction stream, tasks arrive sequentially with heterogeneous formats. Under the shared adaptation paradigm, unified optimization entangles gradient flows, triggering instruction-level interference. Resolving this requires grouping instructions by semantic intent, isolating conflicting updates to prevent cross-format interference.

To implement this, we employ a frozen CLIP~\citep{radford2021learning} text encoder $\psi(\cdot)$ to extract instruction embeddings, yielding normalized vectors $\mathbf{e}(q) = \psi(q)/\|\psi(q)\|_2$. We maintain a cluster center $\mathbf{c}_g$ for each semantic group by averaging the instruction embeddings $\mathbf{e}(q)$ routed to it, while tracking the corresponding sample count $n_g$. Semantic affinity to that group is subsequently measured via cosine similarity $s_g(q) = \frac{\mathbf{c}_g^\top \mathbf{e}(q)}{\|\mathbf{c}_g\|_2}$.

Absolute similarity scores fluctuate unpredictably under distribution shifts. We instead adopt a relative confidence margin. Let $s_{(1)}(q)$ and $s_{(2)}(q)$ denote the highest and second-highest cosine similarities. The assignment margin is:
\begin{equation}
\Delta(q) = s_{(1)}(q) - s_{(2)}(q).
\label{eq:assignment_margin}
\end{equation}
If $\Delta(q) > \theta$, $q$ is confidently assigned to the best-matching cluster $g^* = g_{(1)}$; otherwise, we instantiate a new center $\mathbf{c}_{\text{new}} \leftarrow \mathbf{e}(q)$. Then, the matched cluster's center is updated via:
\begin{equation}
\mathbf{c}_{g^*} \leftarrow \frac{n_{g^*} \mathbf{c}_{g^*} + \mathbf{e}(q)}{n_{g^*} + 1}, \quad n_{g^*} \leftarrow n_{g^*} + 1.
\label{eq:center_update}
\end{equation}
This applies a diminishing step size that naturally attenuates transient noise as clusters expand, ensuring routing boundaries evolve smoothly. 

\noindent\textbf{Discussion.} By grouping instructions with semantic similarity, we protect output conventions, enabling the ability to be safely reused within groups.

\subsection{Adaptive-Rank Expert Instantiation}
\label{subsec:orthogonal_instantiation}
As new tasks arrive continuously, allocating fixed-size experts for each task induces unbounded parameter expansion, whereas allocating no additional experts causes existing parameters to be overwritten, triggering catastrophic forgetting. Fig.~\ref{fig:pre2} reveals that most visual capabilities required by new tasks are largely covered by existing experts, leaving only a small capability gap. Motivated by this, we argue that parameter allocation should exclusively target this gap, reframing the core challenge as dynamically determining the minimal capacity required to encode incremental task knowledge.

To probe the capability gaps, we introduce a warm-up phase that accumulates initial gradient updates into a temporary LoRA buffer $\mathbf{{W}_{\mathcal{B}}}$. As the raw update $\Delta\mathbf{{W}_{\mathcal{B}}}$ mixes task-specific signals with shared gradient components, we apply Singular Value Decomposition to isolate the principal directions that capture the unmet capability gaps:
\begin{equation}
\Delta{\mathbf{{W}_{\mathcal{B}}}} = \sum_{p=1}^{P} \sigma_p \mathbf{u}_p \mathbf{v}_p^\top, \quad \sigma_1 \geq \cdots \geq \sigma_P.
\label{eq:svd_decomp}
\end{equation}
This spectral representation explicitly ranks directional components by geometric significance. To identify the novel correction signals, we implement a two-stage geometric filtering pipeline:

(1) \textit{Spectral Thresholding}: Directions with $\sigma_p \leq \tau_{\mathrm{alloc}}$ are discarded to filter optimization noise, yielding the candidate set $\mathcal{C}$.

(2) \textit{Historical Subspace Orthogonality Filtering}:
In group $g^*$, to prevent allocating parameters to already mastered subspaces, we quantify each candidate's novelty against the expert pool $\mathcal{E}_{g^*}$. 
Since the rows of each down-projection matrix $\mathbf{A}_i$ span the input directions actively modified by expert $i$, we horizontally concatenate their transposes to characterize the union of previously optimized subspaces, yielding the unified historical representation:
\begin{equation}
\mathbf{M}_{\mathrm{hist}} = \big[ \mathbf{A}_1^\top, \dots, \mathbf{A}_{|\mathcal{E}_{g^*}|}^\top \big] \in \mathbb{R}^{d \times R_{\mathrm{hist}}},
\end{equation}
where $R_{\mathrm{hist}}=\sum_{i=1}^{|\mathcal{E}_{g^*}|} r_i$, $r_i$ denotes the number of rows of $\mathbf{A}_i$.
We then extract an orthonormal basis $\mathbf{Q}$ for $\mathrm{col}(\mathbf{M}_{\mathrm{hist}})$ via QR decomposition~\citep{ye2005two}. Geometrically, $\mathbf{Q}\mathbf{Q}^\top$ acts as an orthogonal projector onto the historical subspace. By applying the complementary operator $(\mathbf{I} - \mathbf{Q}\mathbf{Q}^\top)$ to each candidate direction $\mathbf{v}_p \in \mathcal{C}$, we isolate the component orthogonal to established directions. The magnitude of this residual quantifies geometric independence, reflecting the novel optimization capacity $\mathbf{v}_p$ introduces. To eliminate scale dependence, we normalize this by $\|\mathbf{v}_p\|_2$, yielding a bounded orthogonality score:
\begin{equation}
\alpha_p = \frac{\|(\mathbf{I} - \mathbf{Q}\mathbf{Q}^\top)\mathbf{v}_p\|_2}{\|\mathbf{v}_p\|_2 }.
\label{eq:orthogonality_score}
\end{equation}
Only directions satisfying $\alpha_p > \tau_{\mathrm{orth}}$ are retained in $\mathcal{S}$ to reconstruct the final expert parameters:
\begin{equation}
\mathbf{A}_t = \boldsymbol{\Sigma}_{\mathcal{S}}^{1/2}\tilde{\mathbf{V}}_{\mathcal{S}}^\top, \qquad
\mathbf{B}_t = \mathbf{U}_{\mathcal{S}}\boldsymbol{\Sigma}_{\mathcal{S}}^{1/2},
\label{eq:expert_reconstruction}
\end{equation}
where $\mathbf{U}_{\mathcal{S}}$ and $\tilde{\mathbf{V}}_{\mathcal{S}}$ denote the retained singular vectors, and $\boldsymbol{\Sigma}_{\mathcal{S}}$ their corresponding singular values. 
\noindent\textbf{Discussion.} 
By filtering out directions already covered by historical experts, we only allocate parameters to capability gaps, bounding model growth while preserving established representations.

\subsection{Centroid-Guided Orthogonal Learning}
\label{subsec:centroid_routing}
With experts instantiated, another problem arises: how to achieve stable intra-group reuse? Unlike conventional MoE with a static expert pool, MCIT must continuously expand the pool as tasks arrive. This introduces two bottlenecks: First, continuous updates to routing matrices induce boundary drift, creating a training-inference mismatch in expert activation. Second, optimization of subsequent experts within a shared parameter space compounds overlapping feature directions. As the pool expands, repeated updates inflate the magnitude of these directions, skewing MoE aggregation and suppressing novel corrections.

To address these challenges, we assign a dedicated learnable centroid $\mathbf{w}_k \in \mathbb{R}^{d_v}$ to each expert, collectively forming $\mathcal{W} = \{\mathbf{w}_k\}$. These centroids act as geometric anchors that directly govern expert activation. Upon task completion, each centroid is permanently frozen to lock its decision boundary against future shifts. For an incoming visual input $v_i$, we extract its aggregated representation via the MLLM's frozen vision encoder $\phi(\cdot)$ as $\boldsymbol{\xi}_i = \mathrm{MeanPool}(\phi(v_i)) \in \mathbb{R}^{d_v}$. The routing function $\rho$ then computes a proximity score $\rho_{i,k} = \rho(\boldsymbol{\xi}_i, \mathbf{w}_k)$ for every centroid. Within group $g^*$, these scores are normalized across the expert pool $\mathcal{K}_{g^*}$ to yield the routing weights:
\begin{equation}
\omega_{i,k} = \frac{\rho_{i,k}}{\sum_{j \in \mathcal{K}_{g^*}} \rho_{i,j}}, \quad \forall k \in \mathcal{K}_{g^*}.
\label{eq:routing_weights}
\end{equation}
Balancing historical reuse with novel learning requires the routing function to maintain high activation for queries near known centroids while smoothly decaying for distant ones. We achieve this via a Gaussian radial basis kernel:
\begin{equation}
\rho(\boldsymbol{\xi}_i, \mathbf{w}_k) = \exp\!\left(-\frac{\|\boldsymbol{\xi}_i - \mathbf{w}_k\|_2^2}{2\sigma^2}\right),
\label{eq:rbf_score}
\end{equation}
where $\sigma$ controls the bandwidth.

Besides the routing function, proper initialization is also critical. Naive random or zero initialization places centroids off the data manifold, triggering routing collapse. Instead, for a centroid $\mathbf{w}_k$, we anchor it within the local data distribution: during the warm-up phase (Sec.~\ref{subsec:orthogonal_instantiation}), $\mathbf{w}_k$ is initialized as the mean of its corresponding accumulated queries $\boldsymbol{\xi}$. After the warm-up, the centroid is updated via backpropagation for continuous adaptation.

Despite stabilized routing, a structural imbalance persists in expert aggregation. As new experts are instantiated, their updates inherently contain directions shared with historical experts. While new experts learn to cooperate with the existing expert pool during training, historical experts have never adapted to co-activate with future ones. During inference, when an input simultaneously activates both historical and new experts, their outputs along shared directions sum without coordination, causing the effective magnitude of these redundant components to accumulate disproportionately. In contrast, task-specific directions are encoded by a single expert and remain comparatively small. This magnitude imbalance skews the aggregated MoE output toward redundant features, suppressing novel correction signals. To restore balance, new updates must exclusively target unrepresented residuals. We thus constrain $\mathbf{\Delta{W}_{\text{new}}}$ to the orthogonal complement of historical subspaces. Using the precomputed $\mathbf{Q}$ (Sec.~\ref{subsec:orthogonal_instantiation}), we enforce this via an orthogonality penalty for a batch of size $B$:
\begin{equation}
\mathcal{L}_{\mathrm{dec}} = \frac{1}{B} \sum_{i=1}^B \gamma_i \cdot \big\| \Delta\mathbf{{W}_{\text{new}}} (\mathbf{Q}\mathbf{Q}^\top \mathbf{h}_i) \big\|_2^2,
\label{eq:decoupling_loss}
\end{equation}
where $\gamma_i = 1-\omega_{i,\mathrm{new}}$ measures reliance on historical experts,
$\omega_{i,\mathrm{new}}$ is the weight of the newly added expert, and
$\mathbf{h}_i$ denotes the input hidden state.

\noindent\textbf{Discussion.} By routing via visual affinity to stable centroids, we stabilize expert reuse, while an orthogonality penalty guides new updates to unlearned directions, ensuring clean aggregation and learning only what is truly needed.

\subsection{Summary of \hhh}
\label{subsec:summary}
For each task, \hhh{} executes a pipeline: a warm-up phase first accumulates visual queries to initialize centroids $\mathcal{W}$ and extract the historical basis $\mathbf{Q}$, guiding data-driven expert instantiation (Eq.~\ref{eq:expert_reconstruction}). During optimization, centroid-guided routing (Eqs.~\ref{eq:routing_weights}--~\ref{eq:rbf_score}) stabilizes capability reuse, while a penalty (Eq.~\ref{eq:decoupling_loss}) constrains new updates to orthogonal subspaces. Building upon the standard generation loss (Eq.~\ref{eq:nll_loss}), the joint objective is:
\begin{equation}
\mathcal{L} = \mathcal{L}_{\text{NLL}} + \mathcal{L}_{\mathrm{dec}}.
\label{eq:total_loss}
\end{equation}
Upon convergence, historical experts and centroids are frozen before starting the next task. During inference, routing weights are computed directly via Eq.~\ref{eq:routing_weights}. The algorithm is detailed in Appendix~\ref{sec:pseudocode}.

\begin{table*}[t]
\centering
\small
\setlength{\extrarowheight}{1pt}
\setlength{\tabcolsep}{17pt}
\renewcommand{\arraystretch}{1.0}
\caption{Performance on UCIT. The best and second-best results are highlighted in \textbf{bold} and \underline{underline}, respectively.}
\label{tab:ucit}
\resizebox{\linewidth}{!}{
\begin{tabular}{l|cccccc|c}
\hline
\rowcolor{gray!20}
\textbf{Methods} & ImageNet-R & ArxivQA & VizWiz-Caption & IconQA & CLEVR-Math & Flickr30k & Average \\
\hline

\rowcolor{LightPurple!10} 
Zero-shot~\citep{liu2023llava} & 18.88 & 52.62 & 38.75 & 21.25 & 21.12 & 41.44 & 32.34 \\
\rowcolor{LightPurple!20} 
FT-LoRA~\citep{hu2022lora} & 29.33 & 55.30 & 45.51 & 26.13 & 13.07 & \underline{58.07} & 37.90 \\
\rowcolor{LightPurple!30} 
MoE-LoRA~\citep{chen2024coin} & 58.43 & 77.57 & 44.83 & 68.90 & 56.73 & \textbf{58.27} & 60.79 \\
\rowcolor{LightPurple!40} 
Replay-LoRA & 76.93 & 87.07 & 54.31 & 56.43 & 36.40 & 55.94 & 61.18 \\\hdashline
\rowcolor{LightPurple!50} 
CL-MoE~\citep{huai2025cl} & 64.12 & 78.38 & 44.83 & 62.00 & 50.75 & 58.06 & 59.69 \\
\rowcolor{LightPurple!60} 
HiDe-LLaVA~\citep{guo2025hide} & 87.62 & 91.12 & 42.68 & 57.62 & 31.00 & 50.41 & 60.08 \\
\rowcolor{LightPurple!70} 
ModalPrompt~\citep{zeng2025modalprompt} & 80.50 & 90.62 & \underline{60.13} & 63.50 & 55.75 & 57.09 & 67.93 \\
\rowcolor{LightPurple!80} 
DISCO~\citep{guo2025federated} & \underline{88.88} & \textbf{94.25} & 47.52 & 69.50 & 60.75 & 56.32 & 69.54 \\
\rowcolor{LightPurple!90} 
SAME~\citep{xie2026same} & \textbf{89.91} & 91.40 & 55.33 & \underline{77.51} & \underline{68.85} & 55.43 & \underline{73.07} \\
\hdashline
\rowcolor{LightPurple!150} 
\textbf{\hhh{}\ (Ours)} & 87.23 & \underline{92.17} & \textbf{60.55} & \textbf{80.43} & \textbf{70.67} & 57.35 & \textbf{74.73} \\
\hline
\end{tabular}
}
\vspace{-2mm}
\end{table*}

\begin{table*}[t]
\centering
\renewcommand{\arraystretch}{1.0}
\setlength{\extrarowheight}{1pt}
\caption{Performance on TriGap. The best and second-best results are highlighted in \textbf{bold} and \underline{underline}, respectively.}
\label{tab:trigap}
\resizebox{\linewidth}{!}{
\begin{tabular}{l|cccccccccc|c}
\hline
\rowcolor{gray!20}
\textbf{Methods} & PMCVQA & DocVQA & ChartQA & IconQA & InfographicVQA & ArxivQA & Roadside & ChemVQA & FloodNetVQA & CLEVR-Math & Average \\
\hline

\rowcolor{LightPurple!10} 
Zero-shot~\citep{liu2023llava} & 35.40 & 12.68 & 9.36 & 19.27 & 5.06 & 53.77 & 7.40 & 5.30 & 47.41 & 20.37 & 21.60 \\
\rowcolor{LightPurple!20} 
FT-LoRA~\citep{hu2022lora} & 34.20 & 23.32 & 9.84 & 37.07 & 23.53 & 83.83 & 7.00 & 12.70 & 80.31 & 60.27 & 37.21 \\
\rowcolor{LightPurple!30} 
Replay-LoRA & 33.70 & 33.95 & 14.00 & 46.67 & 28.97 & 75.57 & 9.40 & 15.90 & 73.81 & 58.80 & 39.08 \\
\rowcolor{LightPurple!40} 
MoE-LoRA~\citep{chen2024coin} & 39.03 & 37.49 & 12.44 & 43.43 & 35.17 & 90.90 & 7.93 & 20.70 & \textbf{90.41} & \textbf{67.00} & 44.45 \\\hdashline
\rowcolor{LightPurple!50} 
HiDe-LLaVA~\citep{guo2025hide} & 37.00 & 33.20 & 10.52 & 41.97 & 24.09 & 79.20 & 7.73 & 11.17 & 57.39 & 23.00 & 32.53 \\
\rowcolor{LightPurple!60} 
ModalPrompt~\citep{zeng2025modalprompt} & 38.23 & 38.23 & 11.92 & 44.73 & 37.37 & 84.47 & 10.13 & 12.43 & 71.52 & 52.50 & 40.15 \\
\rowcolor{LightPurple!70} 
CL-MoE~\citep{huai2025cl} & 40.53 & 36.79 & 13.72 & 52.70 & 32.27 & \textbf{93.00} & 7.77 & 18.33 & 80.09 & \underline{65.90} & 44.11 \\
\rowcolor{LightPurple!80} 
SAME~\citep{xie2026same} & 41.60 & \underline{43.87} & \underline{17.56} & \underline{64.03} & \textbf{39.57} & 90.46 & \underline{10.83} & 21.77 & \underline{81.09} & 54.50 & 46.53 \\
\rowcolor{LightPurple!90} 
DISCO~\citep{guo2025federated} & \underline{42.03} & 43.50 & \textbf{18.01} & 63.13 & 38.23 & 91.27 & \underline{11.02} & \underline{22.13} & 80.25 & 55.87 & \underline{46.54} \\\hdashline
\rowcolor{LightPurple!150} 
\textbf{\hhh{}\ (Ours)} & \textbf{43.60} & \textbf{44.33} & 16.92 & \textbf{66.00} & \underline{39.17} & \underline{91.83} & \textbf{12.83} & \textbf{24.27} & 79.49 & 59.53 & \textbf{47.79} \\
\hline
\end{tabular}}
\vspace{-4mm}
\end{table*}

\section{Experiments}
\label{sec:experiment}
\subsection{Implementation Details}
\label{subsec:exp_settings}
\noindent\textbf{Datasets.}
We evaluate \hhh{} on two widely used benchmarks. \textbf{UCIT}~\citep{guo2025hide} comprises six tasks: ArxivQA~\citep{li2024arxiv}, CLEVR-Math~\citep{lindstrom2022clevr}, IconQA~\citep{lu2021iconqa}, ImageNet-R~\citep{imagenet-R}, VizWiz-Caption~\citep{gurari2018vizwiz}, and Flickr30k~\citep{plummer2015flickr30k}. \textbf{TriGap}~\citep{xie2026same} encompasses ten tasks: PMCVQA~\citep{zhang2023pmcvqa}, DocVQA~\citep{mathew2021docvqa}, ChartQA~\citep{masry2022chartqa}, IconQA~\citep{lu2021iconqa}, InfographicVQA~\citep{mathew2022infographicvqa}, ArxivQA~\citep{li2024arxiv}, Roadside~\citep{guan2026roadscenevqa}, ChemVQA~\citep{chemvqa}, FloodNetVQA~\citep{rahnemoonfar2021floodnet}, and CLEVR-Math~\citep{lindstrom2022clevr}.

\noindent\textbf{Compared Methods.}
We compare \hhh{} with state-of-the-art MCIT approaches, encompassing two representative structural paradigms: 
(1) \textit{Shared adaptation}: including FT-LoRA, Replay-LoRA, MoE-LoRA~\citep{chen2024coin}, CL-MoE~\citep{huai2025cl}, and SAME~\cite{xie2026same}; 
and (2) \textit{Parameter isolation}: including HiDe-LLaVA~\citep{guo2025hide}, DISCO~\citep{guo2025federated}, and ModalPrompt~\citep{zeng2025modalprompt}. 
We additionally report Zero-shot~\cite{liu2023llava} performance as a lower-bound reference. 
To ensure a fair comparison, all methods adhere to identical backbone architectures and data protocols.

\noindent\textbf{Training Details.}
All experiments are conducted on 4 NVIDIA RTX 5090 GPUs. We follow Prism~\citep{tang2026prism} to conduct all experiments.
We use LLaVA-v1.5-7B~\citep{liu2023llava} as the backbone MLLM and employ the text tower of CLIP-L/14-336~\citep{radford2021learning} to extract instruction embeddings for semantic routing. We inject LoRA adapters into the attention projections of the LLM. Each task is trained for $1$ epoch using the AdamW optimizer with a learning rate of $2 \times 10^{-4}$ and a linear warm-up ratio of $0.03$. A batch size of $8$ is used across all benchmarks. For \hhh{}, hyperparameters are configured as follows: semantic routing margin $\theta = 0.1$, RBF kernel bandwidth $\sigma = 2$, spectral allocation threshold $\tau_{\mathrm{alloc}} = 0.08$, and orthogonality filtering threshold $\tau_{\mathrm{orth}} = 0.99$. For the warm-up stage in Sec.~\ref{subsec:orthogonal_instantiation}, we allocate $5\%$ of training steps and initialize the rank of the temporary buffer $\mathbf{W}_{\mathcal{B}}$ to 48. All other implementation details are provided in Appendix~\ref{app:implementation_details}.

\noindent\textbf{Evaluation Metrics.}
Following \citet{zhou2024class,clmetrics}, we report average final performance $\Bar{\mathcal{A}}=\frac{1}{T}\sum_{s=1}^{T}\mathcal{A}_{s,T}$, where $\mathcal{A}_{s,T}$ denotes accuracy on task $s$ after learning all tasks.

\subsection{Main Results} 
We evaluate \hhh{} on the UCIT and TriGap (Tabs.~\ref{tab:ucit}, \ref{tab:trigap}). 
On the 10-task TriGap benchmark, \hhh{} achieves an average accuracy of {47.79\%}, surpassing the strongest baseline DISCO by 1.25\%. The consistent improvements across heterogeneous domains (e.g., 2.14\% on ChemVQA) reveal that as the task stream lengthens, {adaptive-rank instantiation} effectively curbs parameter redundancy by allocating capacity only to unmet capability gaps, while orthogonality filtering ensures new experts complement rather than overwrite established knowledge. 
On UCIT, \hhh{} attains {74.73\%}, substantially outperforming shared-adaptation baselines like MoE-LoRA by 13.94\%. This directly validates the necessity of isolating instruction formats: by clustering semantically similar prompts, our routing mechanism neutralizes the gradient interference that typically degrades performance in unified parameter spaces.

\subsection{Ablation Study}
We conduct an ablation study on the UCIT benchmark. For variants without adaptive-rank instantiation, the rank is fixed to the average rank of the full \hhh{} to ensure parameter-aligned evaluation.

\begin{table}[t]
\centering
\footnotesize
\setlength{\tabcolsep}{5pt}
\renewcommand{\arraystretch}{1.0}
\caption{Component-wise Ablation.}
\label{tab:ablation_main}
\begin{tabular}{@{}>{\raggedright\arraybackslash}p{5.2cm}|c@{}}
\hline
\textbf{Variant} & \textbf{Avg. Acc (\%)} \\
\hline
Baseline & 53.31 \\
w/ Group & 66.38 \\
w/ Group + Adaptive & 68.92 \\
w/ Group + Adaptive + Centroid (\hhh{}) & \textbf{74.73} \\
\hline
\end{tabular}
\vspace{-3mm}
\end{table}

\begin{table}[t]
\centering
\footnotesize
\setlength{\tabcolsep}{5pt}
\renewcommand{\arraystretch}{1.0}
\caption{Ablation on design choices.}
\label{tab:ablation_design}
\begin{tabular}{@{}>{\raggedright\arraybackslash}p{5.2cm}|c@{}}
\hline
\textbf{Variant} & \textbf{Avg. Acc (\%)} \\
\hline
\textbf{\hhh{}\ (Full)} & \textbf{74.73} \\
Gaussian RBF  $\rightarrow$ cosine similarity & 55.29 \\
warm-up centroid init $\rightarrow$ zero init & 57.29 \\
\hline
\end{tabular}
\vspace{-4mm}
\end{table}

\noindent\textbf{Component-wise Ablation.} 
Tab.~\ref{tab:ablation_main} details step-wise gains. \textbf{Baseline} denotes MoE-LoRA with a fixed expert rank of $r=8$, deliberately aligned with \hhh{}'s average rank on UCIT ($7.94$) for parameter-equivalent comparison (In Tab.~\ref{tab:ucit} MoE-LoRA uses $r=16$). \textit{Group-isolated routing} (+13.07\%) neutralizes format interference via semantic clustering. \textit{Adaptive-rank instantiation} (+2.54\%) targets capability gaps through orthogonality filtering, improving parameter efficiency. \textit{Centroid-guided learning} (+5.81\%) stabilizes expert activation and resolves aggregation imbalance via RBF routing and orthogonal constraints.

\noindent\textbf{Design Choices.}  Tab.~\ref{tab:ablation_design} validates key components via substitution with simpler alternatives. Replacing Gaussian RBF scoring with cosine similarity reduces accuracy by 19.44\%, confirming that RBF's distance-decay property is essential for stable expert activation. Replacing warm-up centroid initialization with zero initialization reduces performance by 17.44\% due to off-manifold instability. 

\subsection{Efficiency Analysis}
We evaluate efficiency from two perspectives: trainable parameters and computational cost. 

\noindent\textbf{Parameter Efficiency.}
Figure~\ref{fig:params} compares parameter counts on TriGap. By allocating parameters only to uncovered capability gaps, \hhh{} achieves the highest average performance with fewer trainable parameters than competing methods.

\begin{figure}[t]
  \centering
  \includegraphics[width=0.8\columnwidth]{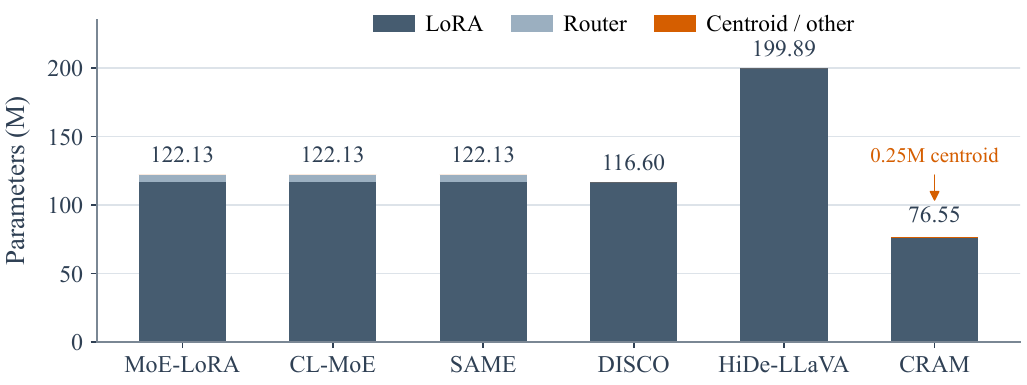}
  \caption{Parameter comparison on TriGap.}
  \label{fig:params}
  \vspace{-4mm}
\end{figure}

\noindent\textbf{Computational Efficiency.}
Figure~\ref{fig:train_time} compares the training time of \hhh{} and MoE-LoRA (Baseline) on UCIT. Although \hhh{} introduces additional computation for SVD decomposition and orthogonality filtering, its training-time overhead remains modest, as adaptive-rank instantiation is performed only when a new expert is created. Table~\ref{tab:efficiency} further reports peak training memory and inference throughput. \hhh{} incurs only a modest memory overhead and achieves the highest inference throughput, as routing and expert activation are restricted to the selected semantic group.

\begin{figure}[t]
  \centering
  \includegraphics[width=0.8\columnwidth]{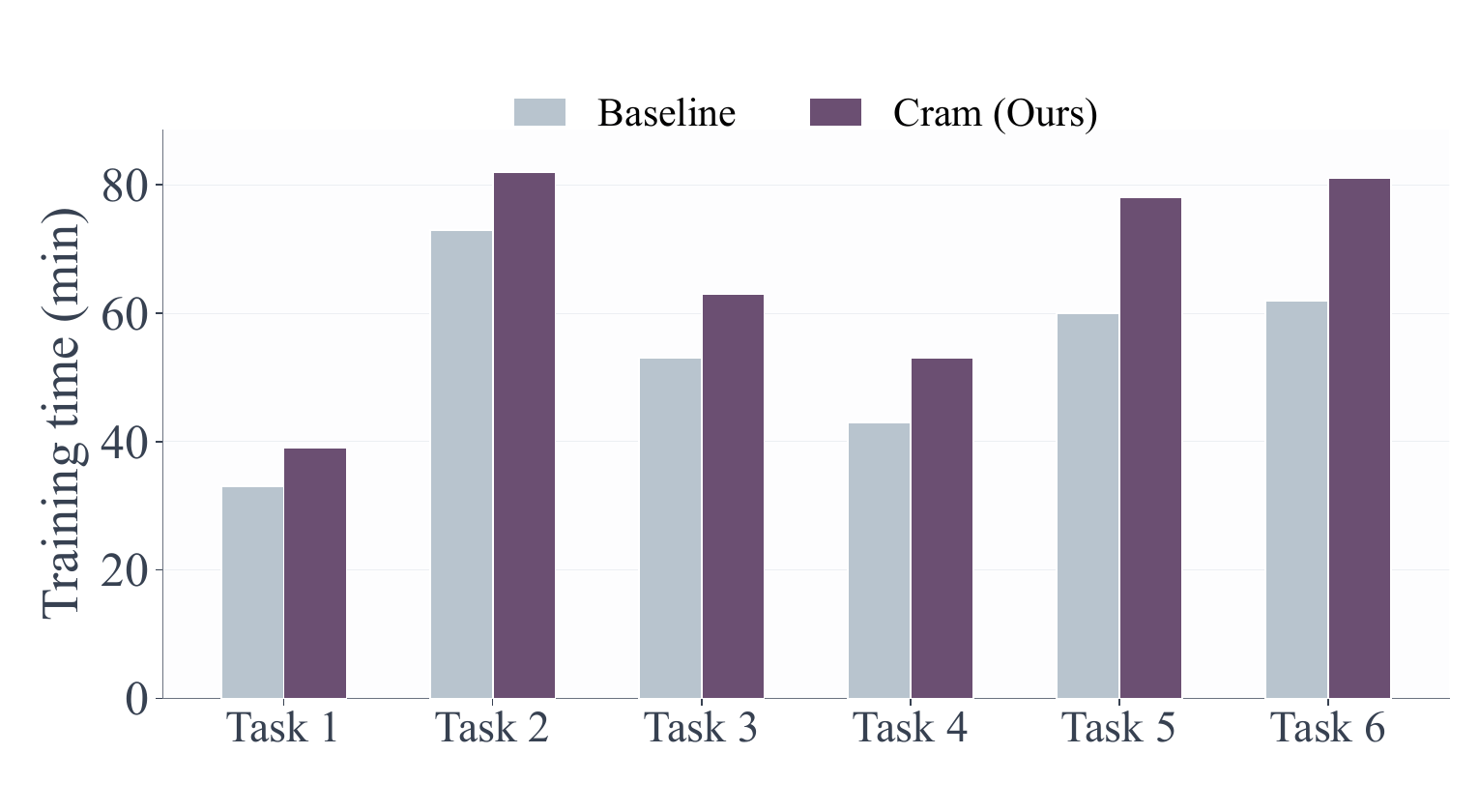}
  \caption{Training time comparison on UCIT.}
  \label{fig:train_time}
    \vspace{-1mm}
\end{figure}

\begin{table}[!t]
\centering
\footnotesize
\setlength{\tabcolsep}{4pt}
\caption{Resource efficiency comparison.}
\label{tab:efficiency}
\begin{tabular*}{\columnwidth}{@{\extracolsep{\fill}}lcc@{}}
\toprule
\textbf{Method} & \textbf{Memory (MiB)} & \textbf{Throughput (samples/s)} \\
\midrule
SAME & 28,971 & 3.96 \\
MoE-LoRA & 28,561 & 3.49 \\
\textbf{\hhh{} (Ours)} & 29,055 & \textbf{4.06} \\
\bottomrule
\end{tabular*}
\end{table}

\subsection{Sensitivity and Robustness Analysis}

We evaluate the robustness of \hhh{} from two perspectives: hyperparameter and prompt variations.

\noindent\textbf{Hyperparameter Sensitivity.}
Using the same hyperparameter configuration on both UCIT and TriGap yields strong performance, suggesting that \hhh{} does not rely on benchmark-specific tuning. As shown in Tab.~\ref{tab:hyperparam}, varying $\theta$, $\sigma$, and the warm-up ratio causes little change in accuracy, demonstrating robustness to hyperparameter choices.

\begin{table}[!t]
\centering
\footnotesize
\setlength{\tabcolsep}{4pt}
\renewcommand{\arraystretch}{1.05}
\caption{Hyperparameter sensitivity on UCIT.}
\label{tab:hyperparam}
\resizebox{\columnwidth}{!}{%
\begin{tabular}{@{}l|ccc|ccc|ccc@{}}
\hline
\textbf{Param.} & \multicolumn{3}{c|}{\textbf{$\theta$}} & \multicolumn{3}{c|}{\textbf{$\sigma$}} & \multicolumn{3}{c@{}}{\textbf{warm-up}} \\
\cline{2-10}
\textbf{Value} & 0.08 & \textbf{0.10} & 0.15 & 1.0 & \textbf{2.0} & 3.0 & 0.03 & \textbf{0.05} & 0.10 \\
\hline
\textbf{Acc (\%)} & 72.4 & \textbf{74.7} & 73.7 & 72.6 & \textbf{74.7} & 71.4 & 70.2 & \textbf{74.7} & 73.2 \\
\hline
\end{tabular}%
}
  \vspace{-2mm}
\end{table}

\begin{figure}[!t]
  \centering
  \includegraphics[width=0.70\columnwidth]{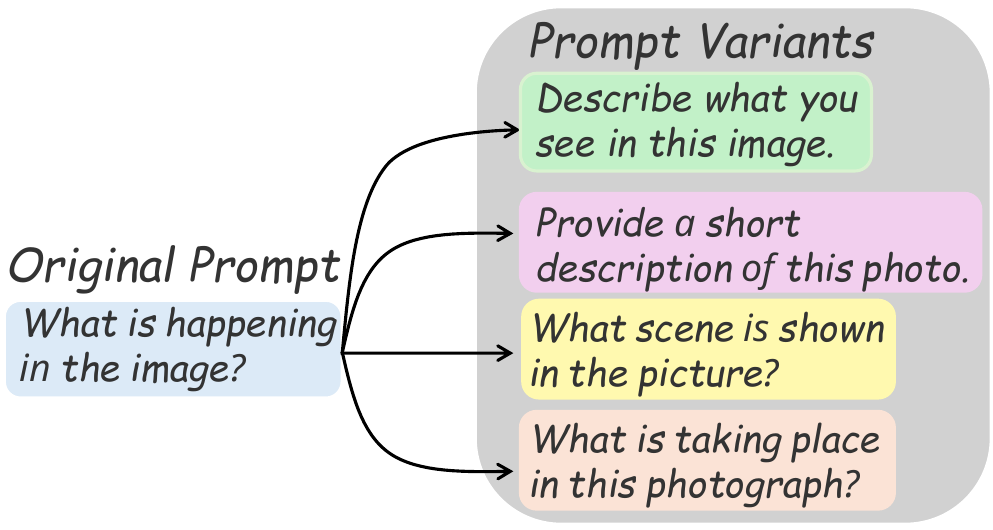}
  \caption{An example of prompt paraphrasing for Flickr30k. The original prompt is rewritten into multiple semantically equivalent variants while preserving the required output format.}
  \label{fig:prompt-variants}
\end{figure}

\begin{table}[!t]
\centering
\setlength{\tabcolsep}{2pt}
\renewcommand{\arraystretch}{1.05}
\caption{Robustness to prompt paraphrasing. T1--T6 correspond to ImageNet-R, ArxivQA, VizWiz-Caption, IconQA, CLEVR-Math, and Flickr30k, respectively. $\Delta$Acc. is reported in percentage points, and Hit Rate is reported in percent.}
\label{tab:paraphrase}
\resizebox{\linewidth}{!}{%
\begin{tabular}{@{}lccccccc@{}}
\toprule
\textbf{Metric} & \textbf{T1} & \textbf{T2} & \textbf{T3} & \textbf{T4} & \textbf{T5} & \textbf{T6} & \textbf{Avg.} \\
\midrule
$\Delta$Acc. & +0.92 & -0.87 & +1.99 & -0.36 & -2.42 & -0.64 & -0.23 \\
Hit Rate(\%) & 100.00 & 100.00 & 100.00 & 99.88 & 100.00 & 100.00 & 99.98 \\
\bottomrule
\end{tabular}%
}
\end{table}

\noindent\textbf{Robustness to Prompt Paraphrasing.}
To test whether semantic routing depends on fixed prompt templates, we generate four semantically equivalent paraphrases for each UCIT instruction while preserving its required output format, as illustrated in Fig.~\ref{fig:prompt-variants}. We define $\Delta$Acc. as the paraphrased accuracy minus the original accuracy. The group hit rate denotes the fraction of paraphrases assigned to the same semantic group as the original instruction. As reported in Table~\ref{tab:paraphrase}, paraphrasing changes the average accuracy by only $-0.23$ points, while maintaining a 99.98\% hit rate, indicating that semantic routing does not rely on fixed prompt templates and remains robust to prompt perturbations.

\subsection{Further Analysis}
\label{subsec:further_analysis}

\begin{figure}[t]
  \centering
  \footnotesize
  \captionsetup[sub]{font=tiny}  
  \begin{subfigure}[b]{0.49\columnwidth}
    \centering
    \includegraphics[width=\linewidth]{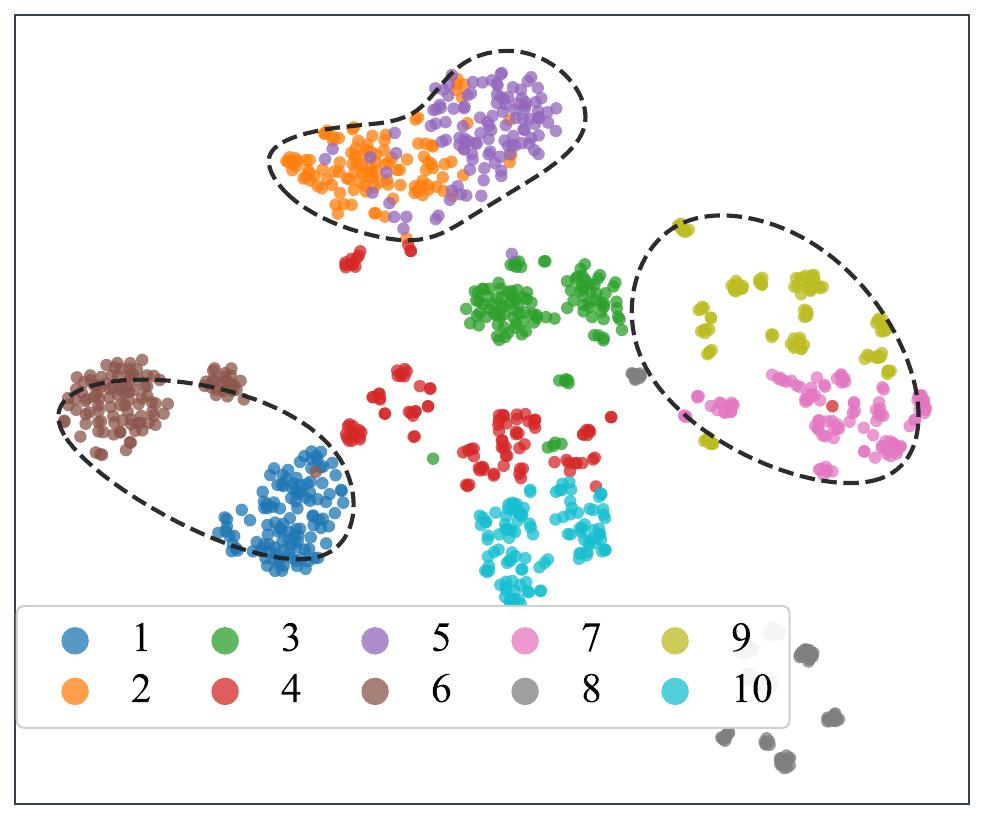}
    \caption{TriGap}
    \label{fig:clustering:a}
  \end{subfigure}
  \hfill
  \begin{subfigure}[b]{0.49\columnwidth}
    \centering
    \includegraphics[width=\linewidth]{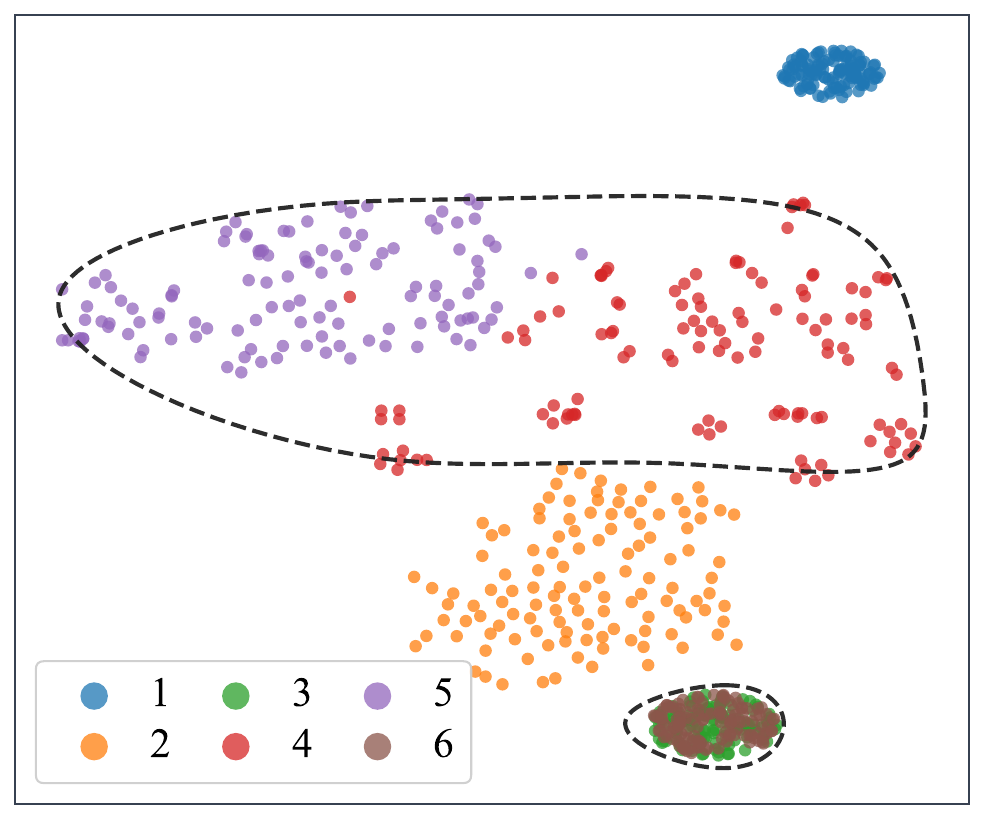}
    \caption{UCIT}
    \label{fig:clustering:b}
  \end{subfigure}
  \caption{Visualization of instruction embeddings on TriGap (a) and UCIT (b). Circled regions indicate instructions assigned to the same group.}
  \label{fig:clustering}
\end{figure}

\noindent\textbf{Clustering Analysis.}
We visualize instruction embeddings using t-SNE~\citep{van2008visualizing} to examine whether the learned groups reflect meaningful instruction structures. As shown in Fig.~\ref{fig:clustering}, clear clusters emerge across both benchmarks: PMCVQA and ArxivQA (multiple-choice) form tight groups, confirming effective grouping. On TriGap, IconQA exhibits diffuse boundaries due to heterogeneous templates, where our relative-confidence margin prevents hard assignment in ambiguous cases, avoiding cross-format interference. This ensures shared adaptation occurs only for unambiguously aligned instructions.

\begin{figure}[t]
  \centering
  \includegraphics[width=1\columnwidth]{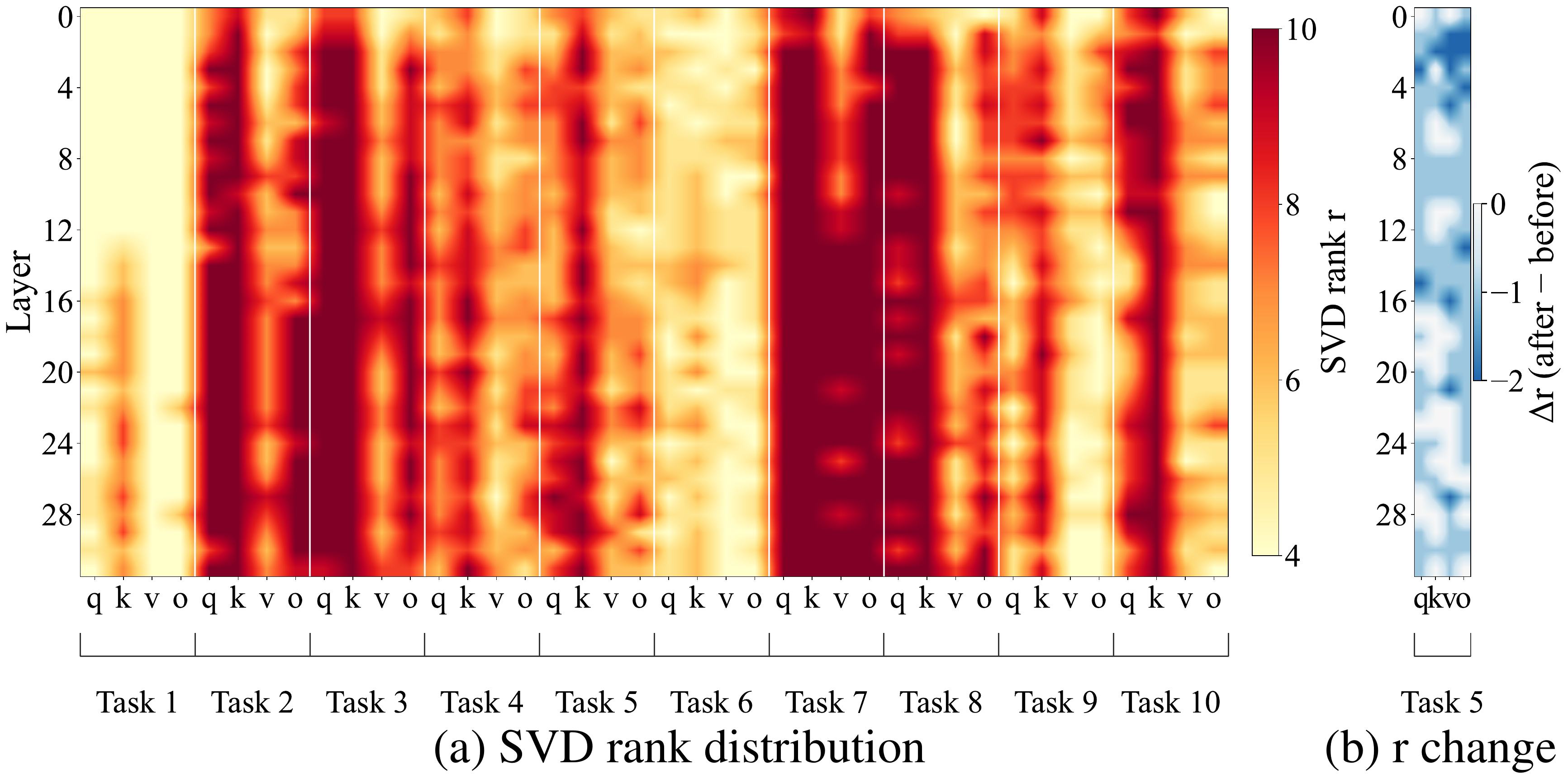}
  \caption{Adaptive rank allocation analysis. (a) Rank distribution across tasks, layers, and modules. (b) Effect of orthogonality filtering: rank allocation for task~5 before/after projecting out historical subspaces.}
  \label{fig:rank_analysis}
\end{figure}

\noindent\textbf{Rank Allocation Analysis.} To analyze how our method dynamically distributes parameter budgets, we examine the rank allocation across tasks and attention modules, as shown in Fig.~\ref{fig:rank_analysis}(a). The results reveal three systematic patterns: (1) \textit{Complexity-driven scaling}, where capacity correlates with task difficulty (e.g., Roadside receives high ranks, while InfographicVQA requires minimal parameters); (2) \textit{Asymmetric module demands}, exhibiting a consistent $K > Q > O > V$ hierarchy across attention projections aligned with their functional roles; and (3) \textit{Intra-cluster reuse}, where tasks within the same semantic cluster receive progressively lower ranks (e.g., Task~5 vs.~Task~2), directly evidencing efficient knowledge transfer. This efficiency stems from the orthogonality filtering, which automatically discards parameter directions already mastered by historical experts. As illustrated in Fig.~\ref{fig:rank_analysis}(b), the filtering predominantly prunes redundant components in lower-layer $V/O$ projections, confirming that historical subspace overlap concentrates on input-proximal visual features. Consequently, \hhh{} preserves foundational representations while dedicating newly allocated capacity to higher-level reasoning pathways.

\begin{figure}[t]
  \centering
  \footnotesize
  \captionsetup[sub]{font=tiny}  
  \begin{subfigure}[b]{0.495\columnwidth}
    \centering
    \includegraphics[width=\linewidth]{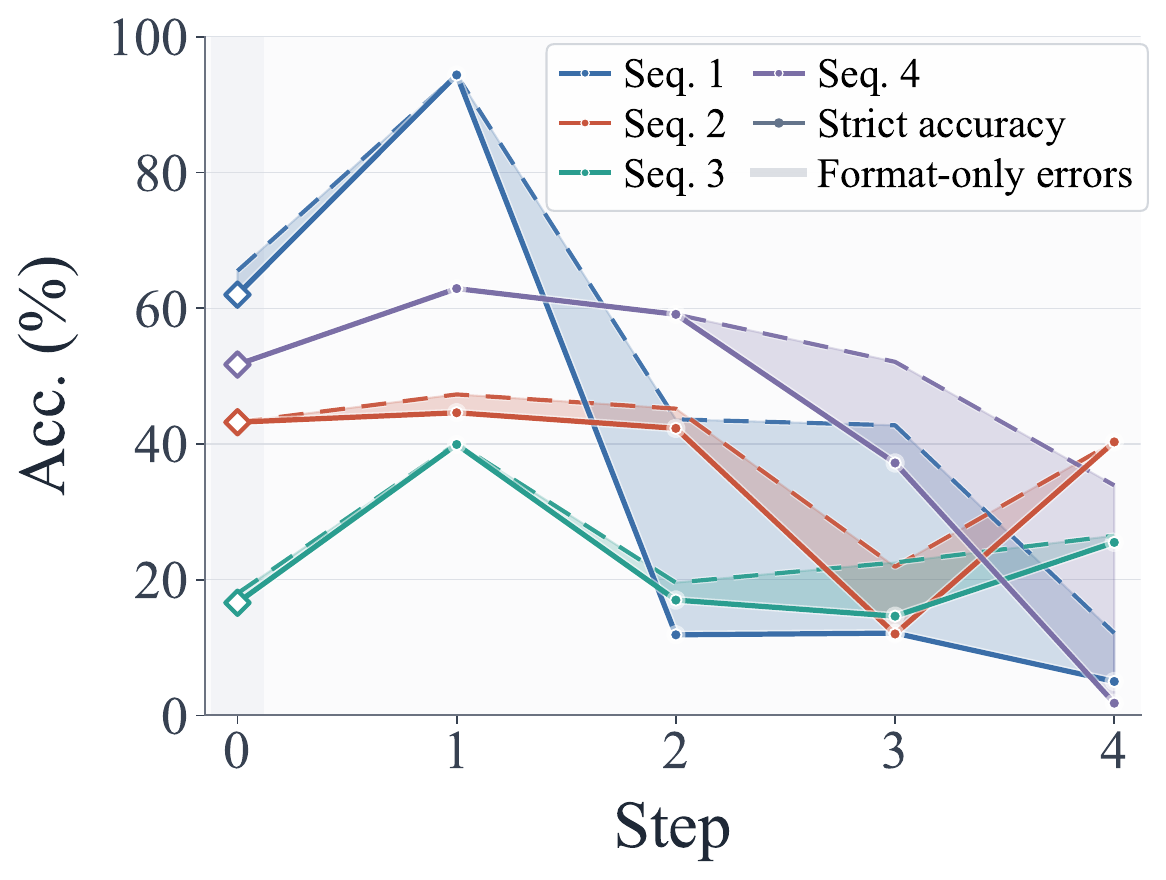}
    \caption{Baseline}
    \label{fig:exp2:a}
  \end{subfigure}
  \hfill
  \begin{subfigure}[b]{0.495\columnwidth}
    \centering
    \includegraphics[width=\linewidth]{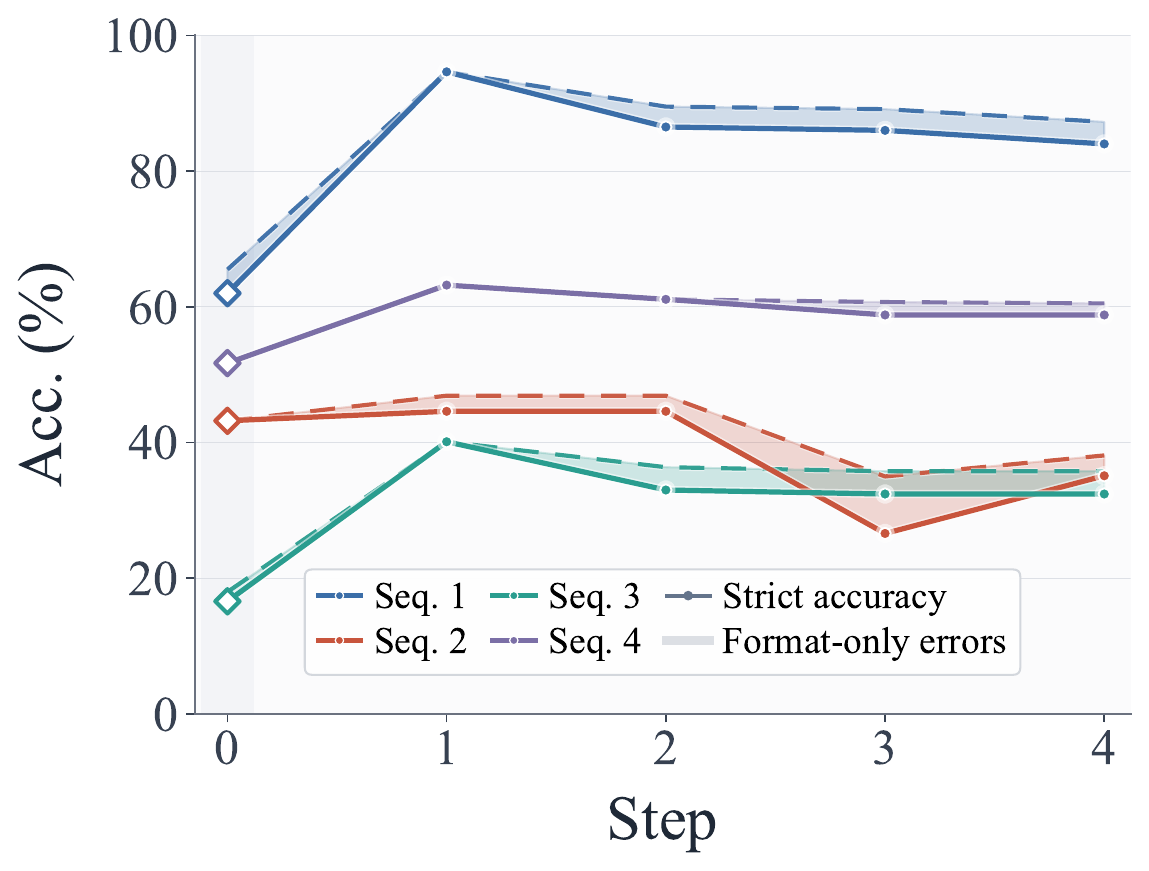}
    \caption{\hhh{} (Ours)}
    \label{fig:exp2:b}
  \end{subfigure}
    \caption{Sequential format adaptation. We construct four instruction-response conventions (\textbf{A}=MCQLetter (letter answer), \textbf{B}=MCQText (text answer), \textbf{C}=Short, \textbf{D}=Yes/No) from sampled data of ArxivQA by varying only the instruction-response format. For each training order (Seq. 1--4: ABCD/BCDA/CDAB/DABC), we evaluate the model on the first trained format after every step. \textit{Strict accuracy} requires both format adherence and content correctness; \textit{content-correct ceiling} counts format-violating but semantically correct responses.}
    \vspace{-1mm}
  \label{fig:exp2}
\end{figure}

\noindent\textbf{Format Interference and Instruction-Following Dynamics.}
To further probe how format heterogeneity interferes with sequential adaptation, we extend the single-step analysis from Fig.~\ref{fig:pre1_1} into a multi-step continual setting, as shown in Fig.~\ref{fig:exp2}. 
\textit{(1) Different instruction formats exhibit varying susceptibility to interference.}
Strongly constrained conventions (e.g., \textit{MCQLetter} and \textit{Yes/No}) degrade sharply when followed by heterogeneous tasks, whereas more flexible formats (e.g., \textit{Short}) show greater resilience.
\textit{(2) Does the model merely overfit to answer formats?} 
Evidence suggests otherwise. For tasks with aligned instruction understanding, such as \textit{MCQLetter} and \textit{MCQText}, they fundamentally share the core capability to extract correct answers. Consequently, training on \textit{MCQLetter} effectively enhances instruction-following capability on \textit{MCQText}, significantly mitigating performance forgetting. 
However, this transfer pattern is inherently \textit{asymmetric}: training on more permissive formats (e.g., \textit{MCQText}) severely disrupts strict boundaries (e.g., \textit{MCQLetter}), rather than yielding positive transfer. This confirms that while the model is genuinely acquiring instruction-following capabilities, its output behavior remains highly sensitive to surface format conventions. Rigid syntactic templates are thus highly vulnerable to \textit{relaxation-induced drift}, which \hhh{} effectively mitigates by isolating format subspaces to sustain stable performance.

\begin{figure}[t]
  \centering
  \begin{subfigure}[b]{0.401\columnwidth}
    \centering
    \includegraphics[width=\linewidth]{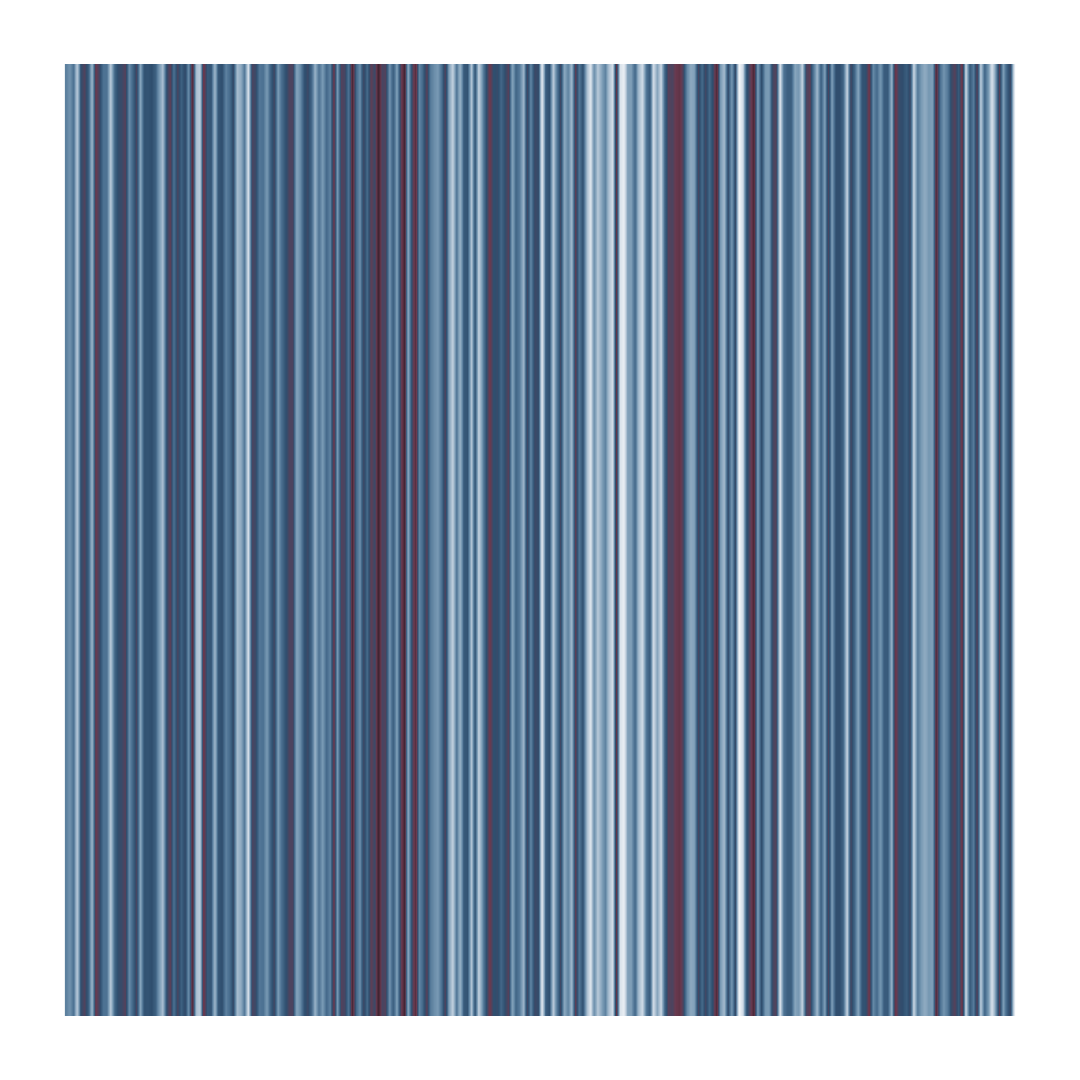}
    \caption{Before}
    \label{fig:overlap:before}
  \end{subfigure}
  \hfill
  \begin{subfigure}[b]{0.571\columnwidth}
    \centering
    \includegraphics[width=\linewidth]{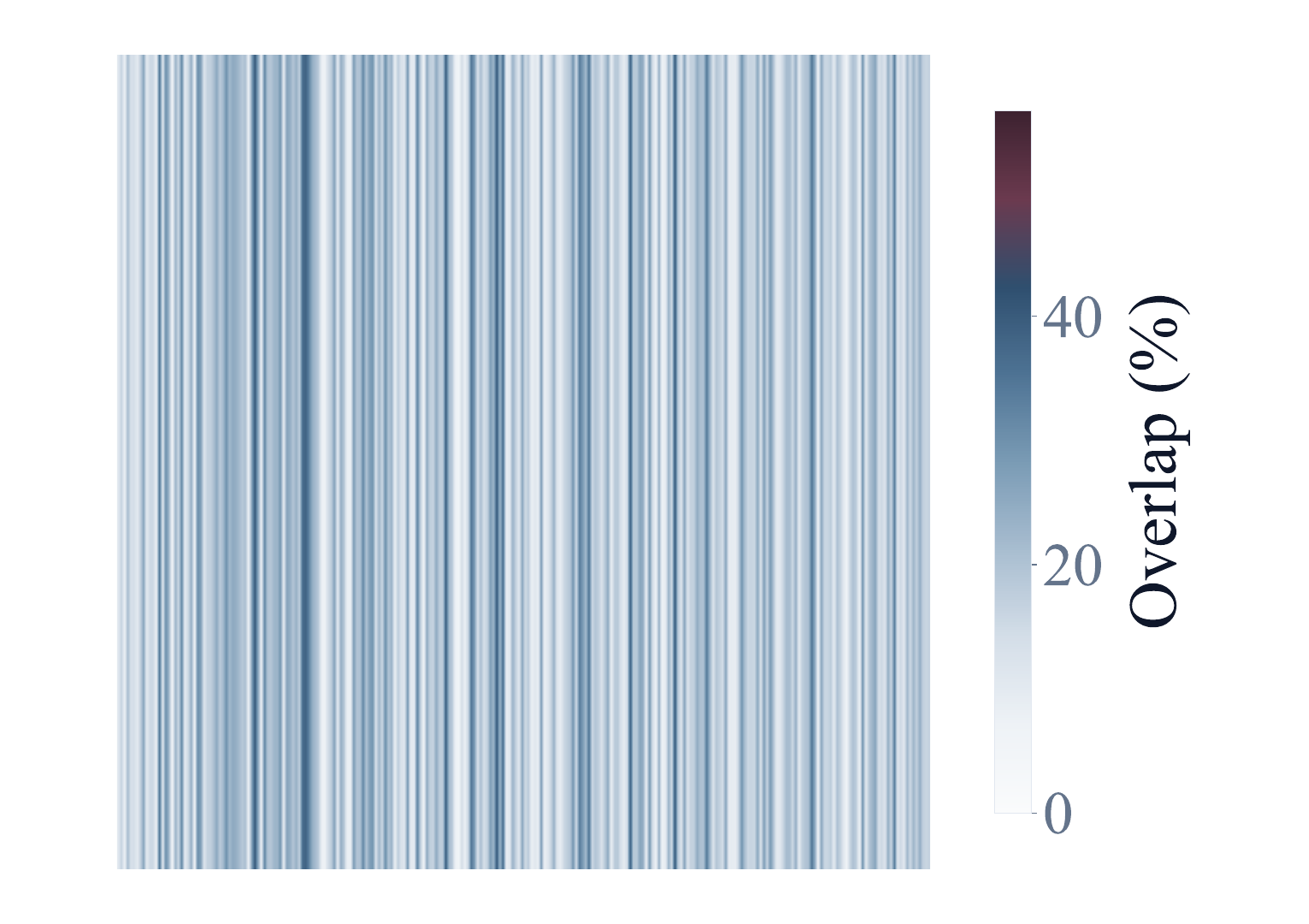}
    \caption{After}
    \label{fig:overlap:after}
  \end{subfigure}
  \caption{Overlap between $\Delta \mathbf{W}$ and a historical expert before and after orthogonality filtering. Each column of $\Delta \mathbf{W}$ is colored by its projection ratio onto the column space of a reference historical expert.}
  \label{fig:overlap_comparison}
\end{figure}

\noindent\textbf{Effect of Orthogonality Filtering.}
Figure~\ref{fig:overlap_comparison} visualizes the projection of each column of $\Delta\mathbf{W}$ onto the column space of a historical expert before and after orthogonality filtering. The average overlap decreases from 36.5\% to 18.1\%, indicating that the filtering substantially reduces alignment between new updates and historical subspaces.

\begin{table}[!t]
\centering
\footnotesize
\setlength{\tabcolsep}{3pt}
\caption{OOD generalization on MMMU after UCIT training.}
\label{tab:mmmu-ood}
\resizebox{\columnwidth}{!}{%
\begin{tabular}{lccccccc}
\toprule
\textbf{Method} & \textbf{Art} & \textbf{Bus.} & \textbf{Sci.}
& \textbf{H\&M} & \textbf{Hum.} & \textbf{T\&E} & \textbf{Avg.} \\
\midrule
Zero-shot & 43.85 & 25.35 & 23.74 & 32.93
& \textbf{48.47} & 28.16 & 33.75 \\
\textbf{\hhh{} (Ours)} & \textbf{44.51} & \textbf{26.19} & \textbf{25.93}
& \textbf{33.76} & 47.34 & \textbf{33.96} & \textbf{35.28} \\
\bottomrule
\end{tabular}%
}
\end{table}

\noindent\textbf{OOD Generalization.}
\label{subsec:ood}
To assess transfer beyond the training tasks, we test UCIT-trained models on MMMU~\citep{yue2024mmmu}. As shown in Table~\ref{tab:mmmu-ood}, \hhh{} achieves an average accuracy of 35.28, outperforming zero-shot by 1.53 points and improving in five of six domains. The largest gains occur in Tech \& Engineering (+5.80) and Science (+2.19), while Humanities declines slightly (-1.13). These results indicate that \hhh{} preserves transferable visual representations and generalizes beyond the UCIT task distribution.

\section{Conclusion}
\label{sec:conclusion}
We propose \hhh{}, which mitigates catastrophic forgetting in MCIT by decoupling instruction isolation from visual adaptation. Through semantic grouping, adaptive instantiation, and centroid-routing, \hhh{} achieves superior retention and parameter efficiency on diverse benchmarks, offering a scalable solution for rehearsal-free MCIT.

\section*{Acknowledgments}
This work is partially supported by the Basic Research Program of Jiangsu (BK20251251), NSFC (62506160), JSTJ2025-147, Fundamental and Interdisciplinary Disciplines Breakthrough Plan of the Ministry of Education of China (No. JYB2025XDXM118), the 111 Center (No. B26023), and the Collaborative Innovation Center of Novel Software Technology and Industrialization. \looseness=-1

\section*{Limitations}
While \hhh{} demonstrates strong performance on current benchmarks, it has not yet been validated on significantly longer task sequences. Extending evaluation to larger-scale continual streams remains an important direction for future work.

\section*{Potential Risks}
This work improves parameter-efficient continual adaptation of Multimodal Large Language Models (MLLMs) under sequential instruction tuning. While our experiments are conducted solely on public academic benchmarks (UCIT, TriGap, and MMMU), such adaptation methods, if applied to unsafe, biased, or privacy-sensitive multimodal data, could be misused to amplify harmful content generation, reinforce dataset biases, or leak sensitive visual information. We do not collect any private user data, and all evaluations are restricted to standard research datasets. 

\section*{Use of AI Assistants}
During the preparation of this manuscript, we only used AI Assistants to assist with grammar checking and language polishing.


\bibliography{custom}

\appendix

\clearpage
\section{Pseudocode}
\label{sec:pseudocode}

\begin{algorithm}[H]
\small
\caption{\hhh{} Training (per task $t$)}
\label{alg:cram}
\begin{algorithmic}[1]
\Require $\mathcal{D}_t$; frozen $(\phi,\pi,f,\psi)$; groups $\{\mathbf{c}_g\}$, pools $\{\mathcal{K}_g\}$
\Ensure New expert $(\mathbf{A}_t,\mathbf{B}_t)$ and centroid $\mathbf{w}_t$

\State \textbf{Level-1 routing.} $\mathbf{e} \leftarrow \psi(q) / \|\psi(q)\|_2$; assign $g^*$ by Eq.~\eqref{eq:assignment_margin}--\eqref{eq:center_update}
\State Allocate expert slot $k_t$ in $\mathcal{K}_{g^*}$; init buffer $\mathbf{W}_{\mathcal{B}}$

\State \textbf{Warm-up.} Train $\mathbf{W}_{\mathcal{B}}$ with $\mathcal{L}_{\mathrm{NLL}}$ (Eq.~\eqref{eq:nll_loss}); accumulate $\boldsymbol{\xi} = \mathrm{MeanPool}(\phi(v))$ to init $\mathbf{w}_{k_t}$

\State \textbf{Instantiation (SVD-Filter).}
\State \hspace{1.5em} Perform SVD: $\Delta\mathbf{W}_{\mathcal{B}} = \mathbf{U}\boldsymbol{\Sigma}\mathbf{V}^\top$ \hfill \Comment{Eq.~\eqref{eq:svd_decomp}}
\State \hspace{1.5em} Keep singular values $\sigma_p > \tau_{\mathrm{alloc}}$
\State \hspace{1.5em} Build $\mathbf{Q}$ from $[\mathbf{A}_1^\top, \mathbf{A}_2^\top, \ldots]$; filter by $\alpha_p$ (Eq.~\eqref{eq:orthogonality_score})
\State \hspace{1.5em} Reconstruct $(\mathbf{A}_{k_t},\mathbf{B}_{k_t})$ via Eq.~\eqref{eq:expert_reconstruction}

\State \textbf{Stable training.} 
\State \hspace{1.5em} $\omega_{i,k}$ from Eqs.~\eqref{eq:routing_weights}--\eqref{eq:rbf_score}; forward
\State \hspace{1.5em} Minimize $\mathcal{L} = \mathcal{L}_{\mathrm{NLL}} + \mathcal{L}_{\mathrm{dec}}$ (Eqs.~\eqref{eq:nll_loss}, \eqref{eq:decoupling_loss}, \eqref{eq:total_loss})

\State Freeze $(\mathbf{A}_{k_t},\mathbf{B}_{k_t},\mathbf{w}_{k_t})$; reset $\mathbf{W}_{\mathcal{B}}$
\end{algorithmic}
\end{algorithm}

\bigskip

\begin{algorithm}[H]
\small
\caption{\hhh{} Inference}
\label{alg:cram_infer}
\begin{algorithmic}[1]
\Require Sample $(v,q)$; frozen $\{\mathbf{c}_g\}$, $\{\mathcal{K}_g\}$, experts $\{(\mathbf{A}_k,\mathbf{B}_k)\}$, centroids $\{\mathbf{w}_k\}$
\Ensure Response $\hat{y}$

\State $\mathbf{e}(q) \leftarrow \psi(q) / \|\psi(q)\|_2$; \quad $g^* \leftarrow \arg\max_g \frac{\mathbf{c}_g^\top \mathbf{e}(q)}{\|\mathbf{c}_g\|_2}$

\State $\boldsymbol{\xi} \leftarrow \mathrm{MeanPool}(\phi(v))$
\State $\omega_k \leftarrow \rho(\boldsymbol{\xi},\mathbf{w}_k) / \sum_{j \in \mathcal{K}_{g^*}} \rho(\boldsymbol{\xi},\mathbf{w}_j)$ \hfill \Comment{Eqs.~\eqref{eq:rbf_score}--\eqref{eq:routing_weights}}
\State $z \leftarrow [\pi(\phi(v)); \mathrm{Emb}(q)]$; decode $\hat{y}$
\end{algorithmic}
\end{algorithm}

\section{Experimental Details of Format Interference}
\label{app:pre1}

To isolate the impact of instruction format heterogeneity, we design a controlled experiment using the ArxivQA~\citep{li2024arxiv} dataset. The goal is to test whether training on a single output convention degrades performance on other conventions despite identical visual and semantic content.

\subsection{Data Construction}

The original ArxivQA samples are multiple-choice questions with four options. We start from the \texttt{Letter} format (the original format), where the model answers with a single letter. From this source, we generate three derived formats by programmatically converting the instruction and answer fields while keeping the image and the semantic correct answer unchanged.

\smallskip
\noindent
\begin{description}
    \item[\textbf{Letter (MCQLetter)}:] The model outputs the option letter.  
        Instruction ends with: \texttt{``Answer with the option's letter from the given choices directly.''}  
        Answer: a single letter, e.g., ``B''.

    \item[\textbf{Text (MCQText)}:] The model outputs the full text of the chosen option.  
        Options are rewritten as \texttt{``A. ...''}, \texttt{``B. ...''}, etc.  
        Instruction: \texttt{``Choose the correct answer and reply with the full text of that choice, exactly as written among the options. Do not use the option letter.''}  
        Answer: the full option text, e.g., \(10^{-1}\) Mpc$^{-1}$.

    \item[\textbf{Short (Short)}:] The model outputs a concise free-form answer. Options are removed. To control difficulty and avoid overly long or complex answers, we retain only samples whose correct option text contains at most two tokens. Instruction: “Answer with the exact wording of the correct choice. Reply with the answer only, without the option label (e.g., A, B, C, or D) or any additional explanation.” Answer: the short text, e.g., “0.48”.

    \item[\textbf{Yes/No (YesNo)}:] The model outputs \texttt{``yes''} or \texttt{``no''}.  
            Two numbered statements (the correct option and a distractor) are presented, exactly one of which is correct. The order is randomly shuffled per sample.  
            To control the difficulty and align with the semantic complexity of the other three formats, we avoid simple true/false questions that directly ask whether a single claim is correct. Instead, this two-statement comparison design requires the model to identify the correct option, ensuring that format conventions rather than answer triviality remain the primary controlled variable.  
            Instruction: \texttt{``Two numbered statements are given about the same question. Exactly one is correct. If statement (1) is correct, reply with exactly: yes. If statement (2) is correct, reply with exactly: no.''}  
            Answer: \texttt{``yes''} or \texttt{``no''}.
\end{description}
\smallskip
After conversion, we align the training and test set sizes across all four formats. For the Short format, the token-length filter (only retaining samples where the correct option text contains at most two tokens) inevitably reduces the number of usable samples. Therefore, we take the intersection of samples that survive this filter and the corresponding samples in the other three formats, ensuring that all four variants are evaluated on exactly the same set of underlying questions. 

\subsection{Training and Evaluation Protocol}

For each format \(f \in \{\text{Letter},\text{Text},\text{Short},\text{Yes/No}\}\), we perform parameter-efficient fine-tuning (PEFT) by inserting LoRA adapters of rank \(r=8\) into the attention layers (i.e., \(q\_proj\), \(v\_proj\), \(k\_proj\), \(o\_proj\)) of LLaVA-v1.5-7B. During training, only the LoRA parameters are learnable; the backbone MLLM remains frozen. The model is trained on the training set of that format for one epoch using the standard language modeling loss. After training, we evaluate the model on the test sets of \textit{all four} formats, yielding a \(4\times 4\) accuracy matrix. The zero-shot baseline (LLaVA-v1.5-7B without fine-tuning) is also evaluated on all four test sets.

\subsection{Error Analysis: Format-Wrong Content-Correct (FWCC)}

To distinguish semantic forgetting from pure format violation, we define a sample as \textit{format-wrong but content-correct} (FWCC) if the predicted answer is semantically equivalent to the ground truth but does not adhere to the required output convention of the test format. For each test format, we use a dedicated evaluator that first checks format compliance (e.g., single letter for Letter, exactly ``yes''/``no'' for Yes/No) and, if violated, attempts to extract the semantic content using rules (e.g., mapping option letters to option texts, normalizing case). The FWCC ratio is the proportion of incorrect predictions that are semantically correct but format-violating.

\subsection{Correspondence with Main Paper Figures}

The results are summarised in Fig.~\ref{fig:pre1} in the main paper. Fig.~\ref{fig:pre1}(a) shows the accuracy change relative to zero-shot for each train–test format pair. The diagonal entries (training and testing on the same format) all achieve high accuracy, confirming that each format is learnable in isolation. However, the off-diagonal entries reveal severe degradation: training on any single format leads to a substantial drop in accuracy when testing on other formats, even though the underlying visual and semantic content remain identical. This indicates that the model learns format-specific surface patterns that do not transfer across conventions.

Fig.~\ref{fig:pre1}(b) further reports the FWCC (Format-Wrong Content-Correct) ratio among errors. A high FWCC ratio indicates that the model still retrieves the correct semantic answer but fails to adhere to the required output convention of the test format. Together, these two panels confirm that heterogeneous instruction formats inherently trigger severe interference, acting as a primary driver of catastrophic forgetting in shared-parameter regimes. Importantly, this interference stems from format-surface misalignment rather than loss of task semantics.

\begin{table}[t]
\centering
\scriptsize                      
\setlength{\tabcolsep}{3pt}     
\renewcommand{\arraystretch}{1.1}
\caption{Implementation details.}
\begin{tabular}{@{}l|c|c|c@{}}
\hline
\textbf{Method} & \textbf{Insertion} & \textbf{Epoch} & \textbf{Config. (UCIT/TriGap)} \\
\hline
FT-LoRA            & Attn+FFN               & 1 & 96/80 (rank) \\
Replay-LoRA        & Attn+FFN               & 1 & 96/80 (rank) \\
HiDe-LLaVA         & Attn+FFN ($N$ exp.)    & 1 & 16/8 (rank) \\
CL-MoE             & FFN ($N$ exp.)         & 1 & 16/8 (rank) \\
DISCO              & FFN ($N$ exp.)         & 1 & 16/8 (rank) \\
ModalPrompt        & Soft prompts ($\ell=10$) & 4 & 10 (prefix) \\
MoE-LoRA            & FFN ($N$ exp.)         & 1 & 16/8 (rank) \\
\textbf{CRAM (Ours)} & Attn ($N$ exp.)  & 1 & 7.94/7.63 (rank) \\
\hline
\end{tabular}
\end{table}

\noindent

\section{Implementation details}
\label{app:implementation_details}
For MoE-based baselines (HiDe-LLaVA, CL-MoE, DISCO, MoE-LoRA), the reported rank is the per-expert average: total rank $r$ is divided by the number of tasks $N$ (UCIT: $N=6$, TriGap: $N=10$).
For CRAM, the adaptive allocation mechanism dynamically determines the per-expert rank, resulting in the reported average values (7.94 for UCIT, 7.63 for TriGap).
All methods inject adapters into the attention (Attn) or feed-forward network (FFN) layers as indicated.
ModalPrompt uses soft prompts of length $\ell=10$ instead of low-rank adapters.

\section{SVD Efficiency Analysis}
\label{app:svd_efficiency}

To quantify the computational cost of expert instantiation, we compare SVD with randomized SVD~\cite{Random-SVM} using two power iterations. As shown in Table~\ref{tab:svd_efficiency}, randomized SVD reduces the decomposition time from 409.5~s to 0.67~s, while the performance decreases slightly.

\begin{table}[t]
\centering
\footnotesize
\setlength{\tabcolsep}{4pt}
\caption{Comparison of SVD variants.}
\label{tab:svd_efficiency}
\begin{tabular*}{\columnwidth}{@{\extracolsep{\fill}}lcc@{}}
\toprule
\textbf{Variant} & \textbf{Time (s)} & \textbf{Final Performance $\Bar{\mathcal{A}}$} \\
\midrule
Full SVD & 409.5 & 74.73 \\
\textbf{Randomized SVD} & \textbf{0.67} & 74.03 \\
\bottomrule
\end{tabular*}
\end{table}

\section{Geometric Properties of the Orthogonality Score}
\label{app:ortho_score}
Let $\mathbf{Q} \in \mathbb{R}^{d \times r_{\mathrm{hist}}}$ be an
orthonormal basis for the historical subspace
$\mathcal{V}_{\mathrm{hist}} = \operatorname{col}(\mathbf{M}_{\mathrm{hist}})$,
where $r_{\mathrm{hist}} = \operatorname{rank}(\mathbf{M}_{\mathrm{hist}})$.
Thus, $\mathbf{Q}^\top \mathbf{Q} = \mathbf{I}_{r_{\mathrm{hist}}}$.

Define
\begin{align*}
\mathbf{P}      &= \mathbf{Q}\mathbf{Q}^\top
&&\text{(onto $\mathcal{V}_{\mathrm{hist}}$)}, \\
\mathbf{P}_\perp &= \mathbf{I} - \mathbf{Q}\mathbf{Q}^\top
&&\text{(onto $\mathcal{V}_{\mathrm{hist}}^\perp$)}.
\end{align*}
Both are symmetric, idempotent, and satisfy
$\mathbf{P}\mathbf{P}_\perp = \mathbf{0}$.

For a candidate SVD direction $\mathbf{v}_p \in \mathbb{R}^d$,
the orthogonality score is
\begin{equation}
\alpha_p
= \frac{\| \mathbf{P}_\perp \mathbf{v}_p \|_2}
       {\| \mathbf{v}_p \|_2}
= \frac{\| (\mathbf{I} - \mathbf{Q} \mathbf{Q}^\top) \mathbf{v}_p \|_2}
       {\| \mathbf{v}_p \|_2}.
\label{eq:alpha_def}
\end{equation}

\begin{proposition}[Bounds and Energy Interpretation]
$\alpha_p \in [0, 1]$.
$\alpha_p = 0$ iff $\mathbf{v}_p \in \mathcal{V}_{\mathrm{hist}}$
(fully redundant);
$\alpha_p = 1$ iff $\mathbf{v}_p \perp \mathcal{V}_{\mathrm{hist}}$
(fully novel).
Moreover, $\alpha_p^2$ equals the fraction of
$\|\mathbf{v}_p\|_2^2$ lying in $\mathcal{V}_{\mathrm{hist}}^\perp$.
\end{proposition}

\begin{proof}
Decompose
$\mathbf{v}_p
= \mathbf{P}\mathbf{v}_p + \mathbf{P}_\perp\mathbf{v}_p
= \mathbf{v}_p^\parallel + \mathbf{v}_p^\perp$.
Since
$\langle \mathbf{v}_p^\parallel, \mathbf{v}_p^\perp \rangle = 0$,
\begin{equation}
\|\mathbf{v}_p\|_2^2
= \|\mathbf{v}_p^\parallel\|_2^2 + \|\mathbf{v}_p^\perp\|_2^2.
\label{eq:pythag}
\end{equation}
Dividing by $\|\mathbf{v}_p\|_2^2 \neq 0$ yields
$1 = \|\mathbf{v}_p^\parallel\|_2^2 / \|\mathbf{v}_p\|_2^2
   + \alpha_p^2$,
hence $0 \leq \alpha_p \leq 1$.
Boundary cases:
$\alpha_p = 0$ iff $\mathbf{v}_p^\perp = \mathbf{0}$,
i.e.\ $\mathbf{v}_p \in \operatorname{col}(\mathbf{Q})$;
$\alpha_p = 1$ iff $\mathbf{v}_p^\parallel = \mathbf{0}$,
i.e.\ $\mathbf{v}_p \perp \operatorname{col}(\mathbf{Q})$.
Finally,
$\alpha_p^2 = \|\mathbf{v}_p^\perp\|_2^2 / \|\mathbf{v}_p\|_2^2$
measures the energy fraction in $\mathcal{V}_{\mathrm{hist}}^\perp$.
Thresholding by $\tau_{\mathrm{orth}}$ retains only directions with
$\alpha_p > \tau_{\mathrm{orth}}$,
so at least $\tau_{\mathrm{orth}}^2$ of their energy lies outside
the historical subspace.
\end{proof}

\section{Gradient Dynamics of the Decoupling Loss}
\label{app:dec_loss}

The decoupling loss penalizes the new expert's response to historical
input directions:
\begin{equation}
\mathcal{L}_{\mathrm{dec}}
= \frac{1}{B} \sum_{i=1}^B \gamma_i \,
  \big\| \Delta\mathbf{W}_{\mathrm{new}}
         (\mathbf{Q}\mathbf{Q}^\top \mathbf{h}_i) \big\|_2^2,
\label{eq:dec_loss_def}
\end{equation}
where $\gamma_i = 1-\omega_{i,\mathrm{new}}$ measures reliance on historical experts,
$\omega_{i,\mathrm{new}}$ is the routing weight of the newly added expert, and
$\mathbf{h}_i \in \mathbb{R}^d$ denotes the input hidden state.
Let $\tilde{\mathbf{h}}_i = \mathbf{Q}\mathbf{Q}^\top \mathbf{h}_i$,
so
$\mathcal{L}_{\mathrm{dec}}
= \frac{1}{B} \sum_{i=1}^B \gamma_i \,
  \|\Delta\mathbf{W}_{\mathrm{new}} \tilde{\mathbf{h}}_i\|_2^2$.

\begin{proposition}[Gradient Structure]
The gradient
$\nabla_{\Delta\mathbf{W}_{\mathrm{new}}}\mathcal{L}_{\mathrm{dec}}$
is right-multiplied by $\mathbf{Q}\mathbf{Q}^\top$,
so gradient descent suppresses $\Delta\mathbf{W}_{\mathrm{new}}$
on historical input directions.
\end{proposition}

\begin{proof}
Since $\gamma_i$ depends only on the routing weights, it is treated as
constant with respect to $\Delta\mathbf{W}_{\mathrm{new}}$.
From
$\frac{\partial}{\partial \mathbf{X}}
   \|\mathbf{X}\mathbf{y}\|_2^2
 = 2\mathbf{X}\mathbf{y}\mathbf{y}^\top$,
\begin{multline}
\frac{\partial \mathcal{L}_{\mathrm{dec}}}
     {\partial \Delta\mathbf{W}_{\mathrm{new}}}
= \frac{2}{B} \sum_{i=1}^B \gamma_i \,
  (\Delta\mathbf{W}_{\mathrm{new}}\tilde{\mathbf{h}}_i)
  \tilde{\mathbf{h}}_i^\top
\\
\shoveleft{
{}= \frac{2}{B} \sum_{i=1}^B \gamma_i \,
   (\Delta\mathbf{W}_{\mathrm{new}}
    \mathbf{Q}\mathbf{Q}^\top\mathbf{h}_i)
   \mathbf{h}_i^\top\mathbf{Q}\mathbf{Q}^\top .
}
\label{eq:grad_final}
\end{multline}
Thus, the gradient takes the form
$\mathbf{G}\,\mathbf{Q}\mathbf{Q}^\top$
with $\mathbf{G} \in \mathbb{R}^{m \times d}$.
The gradient descent update
$\Delta\mathbf{W}_{\mathrm{new}}
 \leftarrow
 \Delta\mathbf{W}_{\mathrm{new}}
 - \eta\,\mathbf{G}\,\mathbf{Q}\mathbf{Q}^\top$
modifies only the components of $\Delta\mathbf{W}_{\mathrm{new}}$
that map inputs from the historical subspace
$\operatorname{col}(\mathbf{Q})$.
Minimizing $\mathcal{L}_{\mathrm{dec}}$ drives
$\Delta\mathbf{W}_{\mathrm{new}} \mathbf{Q} \rightarrow \mathbf{0}$,
forcing near-zero outputs for historical directions.
$\mathcal{L}_{\mathrm{dec}}$ acts as a soft regularizer alongside
$\mathcal{L}_{\mathrm{NLL}}$,
penalizing redundant activations without enforcing hard orthogonality.
\end{proof}

\section{Continual-Learning Metrics}
\label{app:cl_metrics}

To complement the final average accuracy, we evaluate \hhh{} using standard continual-learning metrics on UCIT, following prior work~\citep{zhou2024class,clmetrics}. We report average forgetting, backward transfer (BWT), forward transfer (FWT), and average incremental accuracy (AIA). Lower forgetting and higher BWT, FWT, and AIA indicate better performance.

\begin{table}[t]
\centering
\footnotesize
\setlength{\tabcolsep}{4pt}
\caption{Standard continual-learning metrics on UCIT.}
\label{tab:cl_metrics}
\begin{tabular}{lcccc}
\toprule
\textbf{Method} & \textbf{Forgetting $\downarrow$} & \textbf{BWT $\uparrow$} & \textbf{FWT $\uparrow$} & \textbf{AIA $\uparrow$} \\
\midrule
MoE-LoRA & 18.55 & -18.55 & -10.46 & 68.17 \\
SAME & 4.68 & -4.68 & -7.72 & 80.29 \\
DISCO & 7.45 & -7.45 & -6.28 & 82.63 \\
\textbf{\hhh{} (Ours)} & \textbf{1.35} & \textbf{-1.35} & \textbf{2.47} & \textbf{82.85} \\
\bottomrule
\end{tabular}
\end{table}

As shown in Table~\ref{tab:cl_metrics}, \hhh{} achieves the lowest forgetting, the best BWT, and the highest AIA, indicating stronger retention and more stable performance throughout the continual-learning process. \hhh{} also obtains positive FWT, whereas all baselines show negative values, suggesting that previously acquired capabilities facilitate learning on future tasks.

\section{Multi-Seed Stability}
\label{app:multi_seed}

To assess run-to-run stability, we rerun the UCIT experiments with five random seeds under otherwise identical settings. We compare \hhh{} with SAME and DISCO, two of the strongest baselines. The results are reported in terms of average final accuracy.

\begin{table}[t]
\centering
\footnotesize
\setlength{\tabcolsep}{6pt}
\caption{Multi-seed results on UCIT. Mean and standard deviation are computed over five random seeds.}
\label{tab:multi_seed}
\begin{tabular}{lccc}
\toprule
\textbf{Seed} & \textbf{SAME} & \textbf{DISCO} & \textbf{\hhh{}} \\
\midrule
1 & 72.93 & 68.76 & \textbf{74.80} \\
10 & 72.38 & 69.13 & \textbf{75.03} \\
100 & 72.56 & 68.81 & \textbf{74.56} \\
1000 & 71.86 & 68.50 & \textbf{74.41} \\
42 (Original) & 73.07 & 69.54 & \textbf{74.73} \\
\midrule
\textbf{Average} & 72.56 & 68.95 & \textbf{74.71} \\
\textbf{Std.} & 0.48 & 0.40 & \textbf{0.24} \\
\bottomrule
\end{tabular}
\end{table}

As shown in Table~\ref{tab:multi_seed}, \hhh{} achieves the highest accuracy under every random seed, along with the highest average accuracy and the lowest standard deviation. These results indicate that its improvements are stable across runs rather than being attributable to a particular initialization.

\section{Long-Stream Parameter Efficiency}
\label{app:long_stream}

To evaluate parameter growth over a longer task stream, we combine the tasks from UCIT and TriGap into a single sequence of 16 tasks, including several repeated or closely related tasks. Figure~\ref{fig:long_stream_params} shows the cumulative number of trainable parameters as tasks arrive.

CRAM allocates parameters adaptively according to the capability gap of each incoming task, whereas MoE-LoRA allocates a fixed-size expert for every task. As a result, CRAM requires only 89.53M trainable parameters after 16 tasks, compared with 195.42M for MoE-LoRA, corresponding to a 54.2\% reduction. The slower parameter growth of CRAM demonstrates that previously acquired capabilities can be reused without redundant parameter allocation.

\begin{figure}[t]
  \centering
  \includegraphics[width=\columnwidth]{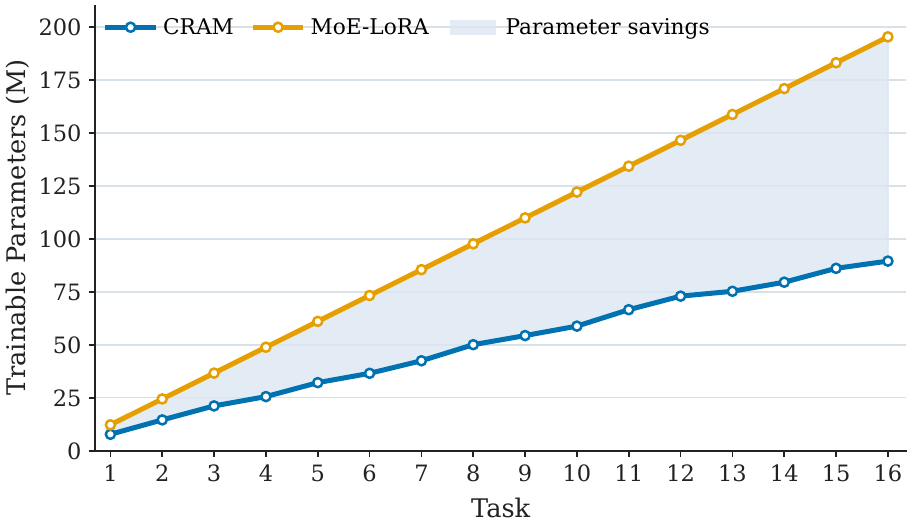}
  \caption{Cumulative trainable parameters over a 16-task stream. The shaded area indicates the parameter savings of \hhh{} compared with MoE-LoRA.}
  \label{fig:long_stream_params}
\end{figure}

\section{Additional Benchmark: MLLM-DCL}
\label{app:mllm_dcl}

To complement the UCIT and TriGap evaluations, we further assess \hhh{} on the more challenging and heterogeneous MLLM-DCL benchmark~\citep{zhao2025mllm}. The results are reported in Table~\ref{tab:mllm_dcl}.

\begin{table}[t]
\centering
\scriptsize
\setlength{\tabcolsep}{2pt}
\caption{Performance on MLLM-DCL. Domain abbreviations follow the benchmark notation.}
\label{tab:mllm_dcl}
\resizebox{\columnwidth}{!}{%
\begin{tabular}{lcccccc}
\toprule
\textbf{Method} & \textbf{RS} & \textbf{Med} & \textbf{AD} & \textbf{Sci} & \textbf{Fin} & \textbf{Avg.} \\
\midrule
Zero-shot & 32.29 & 28.28 & 15.59 & 35.55 & 62.56 & 34.85 \\
FT-LoRA & 69.65 & 41.59 & 25.43 & 40.88 & 87.45 & 53.00 \\
MoE-LoRA & \textbf{77.54} & 41.85 & 27.62 & 40.13 & 86.75 & 54.78 \\
ModalPrompt & 53.63 & 45.68 & 40.77 & 41.81 & 87.82 & 53.94 \\
CL-MoE & 71.34 & 46.84 & 26.33 & 41.17 & 88.74 & 54.88 \\
HiDe-LLaVA & 74.31 & 48.95 & 33.21 & 38.54 & 81.55 & 55.31 \\
DISCO & 76.49 & 44.48 & 44.84 & 46.61 & 89.22 & 60.33 \\
SAME & 72.97 & 48.29 & 43.67 & 48.12 & 87.68 & 60.15 \\
\textbf{\hhh{} (Ours)}
& 73.75 & \textbf{54.89} & \textbf{53.90}
& \textbf{49.91} & \textbf{92.19} & \textbf{64.93} \\
\bottomrule
\end{tabular}%
}
\end{table}

As shown in Table~\ref{tab:mllm_dcl}, \hhh{} achieves the highest average accuracy of 64.93, outperforming the strongest baseline, DISCO, by 4.60 points. It also ranks first in four of the five domains, demonstrating strong generalization under more diverse multimodal continual-learning settings.

\section{Selective Sharing under Coupled Format and Reasoning}
\label{app:coupled_format_reasoning}

\hhh{} does not assume that format learning and visual reasoning are separable within a task. Instead, format and reasoning are learned jointly, while the routing mechanism controls parameter sharing across tasks. Tasks with conflicting instruction conventions can therefore be isolated, whereas tasks with compatible formats and capabilities can reuse existing parameters.

To examine this setting, we construct a CLEVR-Math-based task stream in which format and reasoning are closely coupled:

\begin{description}
    \item[\textbf{Task A:}] ``How many brown matte objects are there? How many blue cylinders are there? What is the total? Answer with the sum number.''

    \item[\textbf{Task B:}] ``How many brown matte objects are there? Answer with a single number.''

    \item[\textbf{Task C:}] ``Are there more brown matte objects than blue cylinders? Answer with yes or no.''
\end{description}

The tasks are learned sequentially in the order A--B--C. Tasks A and B share counting-related capabilities and a numeric output format, whereas Task C differs in both reasoning behavior and output convention. We compare \hhh{}, which groups Tasks A and B while isolating Task C, with an All-in-one strategy that assigns all tasks to a single group.

\begin{table}[t]
\centering
\footnotesize
\setlength{\tabcolsep}{5pt}
\caption{Performance under coupled format and reasoning.}
\label{tab:coupled_format_reasoning}
\begin{tabular}{lcccc}
\toprule
\textbf{Strategy} & \textbf{Task A} & \textbf{Task B} & \textbf{Task C} & \textbf{Avg.} \\
\midrule
All-in-one & 9.0 & 20.0 & 86.6 & 38.5 \\
\textbf{\hhh{} (Ours)} & \textbf{76.2} & \textbf{75.0} & \textbf{90.2} & \textbf{80.5} \\
\bottomrule
\end{tabular}
\end{table}

As shown in Table~\ref{tab:coupled_format_reasoning}, \hhh{} substantially outperforms the All-in-one strategy, demonstrating that selective cross-task isolation remains effective even when instruction format and reasoning behavior are strongly intertwined.

\section{Stability--Plasticity Trade-off}
\label{app:stability_plasticity}

To analyze the effect of the orthogonality constraint, we compare the full \hhh{} model with a variant without this constraint on UCIT. Average Forgetting measures retention of previously learned tasks, Immediate Accuracy measures performance on the current task immediately after training, and FWT measures forward transfer to future tasks relative to the zero-shot baseline~\citep{zhou2024class}.

\begin{table}[t]
\centering
\footnotesize
\setlength{\tabcolsep}{6pt}
\caption{Effect of the orthogonality constraint on stability and plasticity.}
\label{tab:stability_plasticity}
\begin{tabular}{lcc}
\toprule
\textbf{Metric} & \textbf{w/o Orthogonality} & \textbf{\hhh{}} \\
\midrule
Avg. Forgetting $\downarrow$ & 6.73 & \textbf{1.35} \\
Immediate Accuracy $\uparrow$ & \textbf{76.42} & 75.67 \\
FWT $\uparrow$ & \textbf{2.63} & 2.47 \\
\bottomrule
\end{tabular}
\end{table}

As shown in Table~\ref{tab:stability_plasticity}, the orthogonality constraint reduces average forgetting from 6.73 to 1.35, while decreasing Immediate Accuracy and FWT by only 0.75 and 0.16 points, respectively. This result indicates that the constraint substantially improves stability with only a minor loss in plasticity and forward transfer.

\section{Robustness to Task Order}
\label{app:task_order}

While Fig.~\ref{fig:exp2:b} demonstrates \hhh{}'s stability under different sequential format adaptation, we further evaluate its robustness to task ordering on both UCIT and TriGap benchmarks. Specifically, for each benchmark we report results under three distinct task sequences:
\begin{itemize}
    \item \textbf{Original}: the canonical task order defined by the benchmark protocol;
    \item \textbf{Reverse}: the exact reverse of the original sequence, with the last task becoming the first;
    \item \textbf{Random}: a fixed randomly shuffled order:
    \begin{description}
        \item[\textbf{UCIT:}]
        Flickr30k $\rightarrow$ CLEVR-Math $\rightarrow$ ArxivQA $\rightarrow$
        ImageNet-R $\rightarrow$ IconQA $\rightarrow$ VizWiz-Caption

        \item[\textbf{TriGap:}]
        ChemVQA $\rightarrow$ CLEVR-Math $\rightarrow$ PMCVQA $\rightarrow$
        InfographicVQA $\rightarrow$ Roadside $\rightarrow$ DocVQA $\rightarrow$
        ArxivQA $\rightarrow$ FloodNetVQA $\rightarrow$ ChartQA $\rightarrow$ IconQA
    \end{description}
\end{itemize}

\vspace{1mm}

\begin{table}[t]
\centering
\footnotesize
\caption{Task-order robustness of \hhh{} on UCIT. Average accuracy (\%) is reported under the original, reverse, and fixed random task orders.}
\label{tab:task_order_ucit}
\resizebox{\linewidth}{!}{%
\begin{tabular}{@{}lccc@{}}
\toprule
\textbf{Metric} & \textbf{Original} & \textbf{Reverse} & \textbf{Random} \\
\midrule
Avg. Accuracy (\%) & 74.73 & 73.06 & 72.89 \\
\bottomrule
\end{tabular}%
}
\end{table}

\begin{table}[t]
\centering
\footnotesize
\caption{Task-order robustness of \hhh{} on TriGap. The task-order settings follow Table~\ref{tab:task_order_ucit}.}
\label{tab:task_order_trigap}
\resizebox{\linewidth}{!}{%
\begin{tabular}{@{}lccc@{}}
\toprule
\textbf{Metric} & \textbf{Original} & \textbf{Reverse} & \textbf{Random} \\
\midrule
Avg. Accuracy (\%) & 47.79 & 46.97 & 48.01 \\
\bottomrule
\end{tabular}%
}
\end{table}
As shown in Tabs.~\ref{tab:task_order_ucit} and~\ref{tab:task_order_trigap}, \hhh{} exhibits strong robustness to task ordering: performance fluctuations remain marginal across all sequences. This stability stems from two key designs: (1) \textit{group-isolated routing} decouples format-level interference at the semantic level, making expert activation insensitive to the chronological arrival of tasks; and (2) \textit{centroid-guided orthogonal learning} confines new updates to subspaces orthogonal to historical knowledge, preventing order-dependent accumulation of redundant directions. Notably, the minimal degradation under perturbed orders suggests that \hhh{} does not rely on a favorable task curriculum, making it suitable for realistic open-world deployment.

\section{Dataset Details}
\label{app:dataset_details}
Tables~\ref{tab:ucit_details} and~\ref{tab:trigap_details} summarize the diverse datasets and domains used to evaluate continual learning across varied visual inputs and instruction formats.

\begin{table}[t]
\centering
\footnotesize
\caption{UCIT benchmark.}
\label{tab:ucit_details}
\setlength{\tabcolsep}{2pt}
\renewcommand{\arraystretch}{1.1}
\begin{tabular}{@{}p{2.2cm}p{3.2cm}cc@{}}
\toprule
\textbf{Dataset} & \textbf{Domain} & \textbf{Train} & \textbf{Test} \\
\midrule
ImageNet-R   & Artistic           & 24,000 & 3,000 \\
ArxivQA      & Academic figures   & 40,000 & 3,000 \\
VizWiz       & Visual assistance  & 40,000 & 3,000 \\
IconQA       & Icons              & 30,000 & 3,000 \\
CLEVR-Math   & Synthetic reasoning & 40,000 & 3,000 \\
Flickr30k    & Captioning         & 40,000 & 3,000 \\
\bottomrule
\end{tabular}
\end{table}

\begin{table}[t]
\centering
\footnotesize
\caption{TriGap benchmark.}
\label{tab:trigap_details}
\setlength{\tabcolsep}{2pt}
\renewcommand{\arraystretch}{1.1}

\begin{tabular}{@{}p{2.2cm}p{3.2cm}cc@{}}
\toprule
\textbf{Dataset} & \textbf{Domain} & \textbf{Train} & \textbf{Test} \\
\midrule
PMCVQA         & Medical             & 40,000  & 3,000 \\
DocVQA         & Documents           & 30,000  & 5,349 \\
ChartQA        & Charts              & 25,000  & 2,500 \\
IconQA         & Icons               & 10,000  & 3,000 \\
InfographicVQA & Infographics        & 20,000  & 2,801 \\
ArxivQA        & Academic figures    & 10,000  & 3,000 \\
Roadside       & Autonomous driving  & 30,023  & 3,000 \\
ChemVQA        & Molecules           & 15,392  & 3,000 \\
FloodNetVQA    & Disaster scenes     & 5,898   & 1,833 \\
CLEVR          & Synthetic reasoning & 10,000  & 3,000 \\
\bottomrule
\end{tabular}
\end{table}

\section{Comprehensive Study on MCIT Methods}
\label{app:methods}

In this section, we provide details of the methods compared in the main paper. The specifics of each compared method are outlined as follows:

\begin{itemize}
    \item \textbf{Replay-LoRA}: A rehearsal-based baseline that integrates a fixed-size episodic memory buffer with LoRA fine-tuning. During each incremental task, it stores a representative subset of historical data and jointly optimizes the low-rank adapters using both current and replayed samples. This explicit data rehearsal strategy effectively constrains parameter drift and mitigates catastrophic forgetting while preserving parameter efficiency.

    \item \textbf{MoE-LoRA}: This method extends LoRA-based fine-tuning to a Mixture-of-Experts architecture for continual instruction tuning, where each task activates a subset of LoRA experts via a learnable router. However, it remains vulnerable to router drift and expert drift as tasks accumulate.

    \item \textbf{HiDe-LLaVA}: This method hierarchically decouples model adaptation: the top layer undergoes task-specific LoRA expansion with dual-modality anchor matching for expert selection, while lower layers fuse LoRAs across tasks to preserve general knowledge without router training.

    \item \textbf{ModalPrompt}: A prompt-based framework that constructs task-specific prompts and leverages dual-modality guidance for two purposes: prompt fusion during training to transfer knowledge from semantically similar tasks, and prompt selection during inference to control computational complexity. The method maintains inference efficiency by selecting only $k$ relevant prompts from a shared pool regardless of task count.

    \item \textbf{CL-MoE}: A dual momentum Mixture-of-Experts framework designed for continual MLLM adaptation. It introduces a Dual-Router MoE (RMoE) that jointly employs task-level and instance-level routers to robustly allocate global and local experts. Coupled with a Dynamic Momentum MoE (MMoE), it categorizes experts into task-shared and task-specific types and dynamically updates their parameters via a momentum-based interpolation between historical and current knowledge. This mechanism effectively alleviates catastrophic forgetting and enhances both forward and backward transfer without relying on replay buffers.

    \item \textbf{SAME}: SAME stabilizes MoE adaptation via spectral-aware routing (decomposing router updates into task-relevant and history-preserving subspaces) and curvature-aware scaling (regulating expert updates using historical input covariance). It also introduces adaptive expert activation to freeze selected experts during training, reducing redundant computation and cross-task interference. SAME achieves strong performance on MCIT benchmarks and serves as a direct point of comparison for our \hhh{} framework.

    \item \textbf{DISCO}: A federated continual learning framework that employs dynamic knowledge organization (DKO) to store task-specific parameters in a global cache using identity token matching, and subspace selective activation (SSA) to filter irrelevant outputs during inference. While originally designed for federated settings, its continual learning components are applicable to centralized MCIT and are included for comparison.
\end{itemize}

\section{Erratum}
\label{sec:erratum}

\begin{table*}[t]
\centering
\renewcommand{\arraystretch}{1.0}
\setlength{\tabcolsep}{4pt}
\caption{
Corrected performance on TriGap (\%).
Only the ChemVQA scores and the corresponding averages are updated.
The best and second-best results in each column are highlighted
in \textbf{bold} and \underline{underline}, respectively.
Average denotes the unweighted mean over the ten datasets.
}
\label{tab:trigap_erratum}
\resizebox{\textwidth}{!}{
\begin{tabular}{l|cccccccccc|c}
\hline
\textbf{Methods}
& PMCVQA
& DocVQA
& ChartQA
& IconQA
& InfographicVQA
& ArxivQA
& Roadside
& ChemVQA
& FloodNetVQA
& CLEVR
& Average \\
\hline

Zero-shot~\citep{liu2023llava}
& 35.40 & 12.68 & 9.36 & 19.27 & 5.06
& 53.77 & 7.40 & 11.10 & 47.41 & 20.37
& 22.18 \\

FT-LoRA~\citep{hu2022lora}
& 34.20 & 23.32 & 9.84 & 37.07 & 23.53
& 83.83 & 7.00 & 23.43 & 80.31 & 60.27
& 38.28 \\

Replay-LoRA
& 33.70 & 33.95 & 14.00 & 46.67 & 28.97
& 75.57 & 9.40 & 30.70 & 73.81 & 58.80
& 40.56 \\

MoE-LoRA~\citep{chen2024coin}
& 39.03 & 37.49 & 12.44 & 43.43 & 35.17
& 90.90 & 7.93 & 38.43
& \textbf{90.41} & \textbf{67.00}
& 46.22 \\
\hline

HiDe-LLaVA~\citep{guo2025hide}
& 37.00 & 33.20 & 10.52 & 41.97 & 24.09
& 79.20 & 7.73 & 20.83 & 57.39 & 23.00
& 33.49 \\

ModalPrompt~\citep{zeng2025modalprompt}
& 38.23 & 38.23 & 11.92 & 44.73 & 37.37
& 84.47 & 10.13 & 24.50 & 71.52 & 52.50
& 41.36 \\

CL-MoE~\citep{huai2025cl}
& 40.53 & 36.79 & 13.72 & 52.70 & 32.27
& \textbf{93.00} & 7.77 & 34.80
& 80.09 & \underline{65.90}
& 45.76 \\

SAME~\citep{xie2026same}
& 41.60 & \underline{43.87} & \underline{17.56}
& \underline{64.03} & \textbf{39.57}
& 90.46 & 10.83 & 40.77
& \underline{81.09} & 54.50
& 48.43 \\

DISCO~\citep{guo2025federated}
& \underline{42.03} & 43.50 & \textbf{18.01}
& 63.13 & 38.23
& 91.27 & \underline{11.02} & \textbf{46.53}
& 80.25 & 55.87
& \underline{48.98} \\
\hline

\textbf{CRAM (Ours)}
& \textbf{43.60} & \textbf{44.33} & 16.92
& \textbf{66.00} & \underline{39.17}
& \underline{91.83} & \textbf{12.83}
& \underline{45.80}
& 79.49 & 59.53
& \textbf{49.95} \\
\hline
\end{tabular}
}
\end{table*}

For the TriGap experiments, we followed the official evaluation
protocol and used the evaluation code released by PRISM.
However, a reference-matching issue in this official evaluation
code was subsequently reported in the PRISM repository.\footnote{
\url{https://github.com/LAMDA-CL/Prism/issues/7}}
Since our evaluation followed the same pipeline, the reported
ChemVQA scores were also affected.
We therefore re-evaluated the ChemVQA results using the corrected
reference matching and updated the corresponding average scores.

Specifically, the ChemVQA test split contains 3,000 questions but
only 1,576 distinct question IDs, as different questions about
the same molecule image can share an ID.
The original PRISM evaluator retrieved reference answers using
only the question ID.
For annotations sharing an ID, only the last entry was retained,
so predictions for different questions were compared against
the same reference answer.
The corrected evaluator matches each prediction to its reference
using both the question ID and the question text.

Table~\ref{tab:trigap_erratum} presents the corrected results.
Only the ChemVQA scores and the corresponding averages are updated.
CRAM retains its leading overall performance, achieving the highest
average accuracy of 49.95\%, exceeding the strongest baseline,
DISCO (48.98\%), by 0.97 percentage points.
Beyond the highest average, CRAM achieves the best results on four
datasets and ranks among the top two methods on seven of the ten
datasets, demonstrating consistently strong performance across
diverse tasks.
These results confirm that CRAM's overall advantage on TriGap
remains intact under the corrected evaluation.

\end{document}